\documentclass{article}

\usepackage{amsmath,amsthm,amssymb,amsfonts}
\usepackage{stmaryrd}
\usepackage{graphicx}
\usepackage{titlesec}
\usepackage{hyperref}
\usepackage{xspace}
\usepackage{float}
\usepackage{caption}
\usepackage{epstopdf}
\usepackage{fullpage}
\usepackage{color}

\allowdisplaybreaks

\usepackage{authblk}

\newcommand{\bc}[1]{\left\{{#1}\right\}}
\newcommand{\br}[1]{\left({#1}\right)}
\newcommand{\bs}[1]{\left[{#1}\right]}
\newcommand{\abs}[1]{\left| {#1} \right|}

\newcommand{\norm}[1]{\left\| {#1} \right\|}

\newcommand{\Pp}[2]{\underset{#1}{\mathbb{P}}\bs{{#2}}}

\newcommand{\Ee}[2]{\underset{#1}{\mathbb{E}}\bs{{#2}}}

\newtheorem{lemma}{Lemma}
\newtheorem{theorem}{Theorem}
\newtheorem{proposition}{Proposition}

\newtheorem{corollary}{Corollary}

\newtheorem{definition}{Definition}

\DeclareMathOperator*{\argmin}{arg\,min}

\newcommand{\ind}[1]{\llbracket {#1} \rrbracket}

\makeatletter
\newcommand{\newreptheorem}[2]{\newtheorem*{rep@#1}{\rep@title} 
	\newenvironment{rep#1}[1]{\def\rep@title{#2 \ref*{##1}}\begin{rep@#1}}{\end{rep@#1}}
}
\makeatother

\newreptheorem{lemma}{Lemma}
\newreptheorem{theorem}{Theorem}
\newreptheorem{claim}{Claim}
\newreptheorem{proposition}{Proposition}
\newreptheorem{corollary}{Corollary}

\title{Minimax Lower Bounds for Cost Sensitive Classification}

\author[1]{Parameswaran Kamalaruban\thanks{kamalaruban.parameswaran@epfl.ch}}
\author[2]{Robert C.~Williamson\thanks{bob.williamson@anu.edu.au}}
\affil[1]{School of Engineering, \'Ecole Polytechnique F\'ed\'erale de Lausanne}
\affil[2]{Research School of Computer Science, The Australian National University}

\date{\today}

\begin{document}
	
	\maketitle
	
	\begin{abstract}
The cost-sensitive classification problem plays a crucial role in mission-critical machine learning applications, and differs with traditional classification by taking the misclassification costs into consideration. Although being studied extensively in the literature, the fundamental limits of this problem are still not well understood. We investigate the hardness of this problem by extending the standard minimax lower bound of balanced binary classification problem (due to \cite{massart2006risk}), and emphasize the impact of cost terms on the hardness.
\end{abstract}
	
	\section{Introduction}
\label{sec:introduction}
The central problem of this paper is the cost-sensitive binary classification problem, where different costs are associated with different types of mistakes. Several important machine learning applications such as medical decision making, targeted marketing, and intrusion detection can be naturally formalized as cost-sensitive classification setup (\cite{abe2004iterative}). In these domains, the cost of missing a target is much higher than that of a false-positive, and classifiers that do not take misclassification costs into account do not perform well.

The cost-sensitive classification problem has been extensively studied, and people have developed efficient algorithms with provable guarantees on the (generalization) error \cite{chan1998toward, elkan2001foundations, zadrozny2003cost,zhao2015cost, knoll1994cost, brefeld2003support}. These methods primarily take existing classification methods based on empirical risk minimization and try to adapt them in various ways to be sensitive to these misclassification costs.

Despite all these efforts, the understanding of the fundamental limits of this problem is still missing. In this paper, we study the hardness of this problem by obtaining minimax lower bounds. In particular, we are interested in understanding how the cost parameter influences the hardness or complexity of the cost-sensitive classification.

\paragraph{Minimax Lower Bounds}
Understanding the hardness or fundamental limits of a learning problem is important for practice for the following reasons:
\begin{itemize}
	\item They give an estimate on the number of samples required for a good performance of a learning algorithm.
	\item They give an intuition about the quantities and structural properties which are essential for a learning process and therefore about which problems are inherently easier than others.
	\item They quantify the influence of parameters and indicate what prior knowledge is relevant in a learning setting and therefore they guide the analysis, design, and improvement of learning algorithms. 
\end{itemize}
Note that the ``hardness'' here corresponds to lower bounds on sample complexity (and not computational complexity). We demonstrate the hardness of a learning problem (and the instantiations of it) by obtaining lower bounds for its \textit{minimax risk} (formally defined later).

In section~\ref{sec:hard-cost-sense} we review and extend techniques due to \cite{le2012asymptotic} and \cite{assouad1983deux} for obtaining minimax lower bounds for learning problems. Both techniques proceed by reducing the learning problem to an easier hypothesis testing problem  \cite{Tsybakov2009nonparbook,yang1999information,yu1997assouad}, then proving a lower bound on the probability of error in testing problems.

Le Cam's method, in its simplest form, provides lower bounds on the error in simple binary hypothesis testing problems, by using the connection between hypothesis testing and total variation distance. Consider the parameter estimation problem with parameter space $\Theta = \bc{-1,1}^m$ for some $m$, where the objective is to determine every bit of the underlying unknown parameter $\theta \in \Theta$. In that setting, a key result known as \textit{Assouad's lemma} says that the difficulty of estimating the entire bit string $\theta$ is related to the difficulty of estimating each bit of $\theta$ separately, assuming all other bits are already known.

In binary classification problems, the Tsybakov noise condition \cite{tsybakov2004optimal} governs the fraction of data that lies close to the classification boundary. We present a cost-sensitive version of this assumption \eqref{noise-condition}, and then provide minimax lower bounds for cost-sensitive classification under this assumption (Theorem~\ref{main-theo-cost-classi}).

The paper is organized as follows. In Section~\ref{sec:preliminaries} we formally introduce the general learning task, cost-sensitive classification problem, Bayes classifier and minimax risk. In Section~\ref{sec:hard-cost-sense} we consider our main problem - hardness of the cost-sensitive classification problem. Section~\ref{sec:conclusion} concludes with a brief discussion. Proofs of our results are deferred to Appendix~\ref{sec:proofs}.

	\section{Preliminaries and Background}
\label{sec:preliminaries}
This section provides the necessary background on general learning problem, cost-sensitive classification problem, and some decision theoretic notions associated with them.

\subsection{Notation}
\label{subsec:notation}
We require the following notation and definitions. The real numbers are denoted $\mathbb{R}$, and the non-negative reals $\mathbb{R}_{+}$.
%

\paragraph{Probabilities and Expectations}
Let $\Omega$ be a measurable space and let $\mu$ be a probability measure on $\Omega$. $\Omega^n$ denotes the product space $\Omega \times \cdots \times \Omega$ endowed with the product measure $\mu^n$. The notation $\textsf{X} \sim \mu$ means $\textsf{X}$ is randomly drawn according to the distribution $\mu$. $\mathbb{P}_\mu\bs{E}$ and $\Ee{\textsf{X} \sim \mu}{f\br{\textsf{X}}}$ will denote the probability of a statistical event $E$ and the expectation of a random variable $f\br{\textsf{X}}$ with respect to $\mu$ respectively. We will use capital letters $\textsf{X},\textsf{Y},\textsf{Z},\dots$ for random variables and lower-case letters $x,y,z,\dots$ for their observed values in a particular instance. We will denote by $\mathcal{P}\br{\mathcal{X}}$ the set of all probability distributions on an alphabet $\mathcal{X}$.

\paragraph{Metric Spaces}
The Hamming distance on $\mathbb{R}^n$ is defined as
\begin{equation}
\label{eq-hamming-distance}
\rho_\mathrm{Ha}\br{x,x'} := \sum_{i=1}^{n}{\ind{x_i \neq x_i'}} ,
\end{equation}
where $\ind{P}=1$ if $P$ is true and $\ind{P}=0$ otherwise. For a probability measure $\mu$ on a measurable space $\Omega$ and $1 \leq p \leq \infty$, let $L_p \br{\mu}$ be the space of measurable functions on $\Omega$ with a finite norm
\begin{equation}
\label{eq-lp-norm-func}
\norm{f}_{L_p \br{\mu}} := \br{\int{\abs{f}^p d\mu}}^{1/p} .
\end{equation}
$\mathcal{Y}^{\mathcal{X}}$ represents the set of all measurable functions $f:\mathcal{X}\rightarrow\mathcal{Y}$. Define $\bar{c} := 1-c$, for $c \in \bs{0,1}$. We write $[n]:=\{1,\dots,n\}$, and $x \wedge y := \min \br{x,y}$. A mapping $t \mapsto \mathrm{sign}\br{t}$ is defined by
\[
\mathrm{sign}\br{t} = \begin{cases}
1 & \text{if } t \geq 0 \\
-1 & \text{otherwise}
\end{cases} .
\] 

\subsection{General Learning Task}
\label{subsec:general-learning-task}

A \textit{general learning task} in statistical decision theory can be viewed as a two player game between the \textit{decision maker} and \textit{nature} as follows: Given the parameter space $\Theta$, observation space $\mathcal{O}$, and decision space $\mathcal{A}$, and the loss function $\ell : \Theta \times \mathcal{A} \rightarrow \mathbb{R}_+$,
\begin{itemize}
	\item Nature chooses $\theta \in \Theta$, and generates the \emph{iid} data $\textsf{O}_1^n \sim P_\theta^n \in \mathcal{P}\br{\mathcal{O}^n} := \bc{P_\theta^n := P_\theta \times \cdots \times P_\theta : \theta \in \Theta}$, where $P_\theta$ is the distribution determined by the parameter $\theta$,
	\item The decision maker observes the data $\textsf{O}_1^n$, makes her own decision $a \in \mathcal{A}$ (using a stochastic \textit{decision rule}), and incurs loss $\ell\br{\theta,a}$. A stochastic decision rule (denoted by $A:\mathcal{O}^n \rightsquigarrow \mathcal{A}$) is a mapping from the observation space $\mathcal{O}^n$ to the space of probability measures over the action space $\mathcal{A}$.
\end{itemize} 
The general learning task is compactly denoted by the tuple $(\Theta,\mathcal{O},\mathcal{A},\ell)$. Throughout the paper we assume $\Theta$ to be finite and $\mathcal{A}$ to be closed, compact, set in order to provide a clear presentation by avoiding the measure theoretic complexities.

\paragraph{Regret:}
The loss relative to the best action is the \textit{regret} (for any $\theta \in \Theta$ and $a \in \mathcal{A}$)
\begin{equation}
	\label{regret-eq}
	\Delta \ell \br{\theta,a} ~:=~ \ell \br{\theta,a} - \inf_{a' \in \mathcal{A}}{\ell \br{\theta,a'}}.
\end{equation}

\paragraph{Risk:}
The quality of the final action chosen by the decision maker when
she uses the the stochastic decision rule $A:\mathcal{O} \rightsquigarrow \mathcal{A}$ can be evaluated using the notion of \textit{risk}: for any $\theta \in \Theta$: 
\begin{equation}
	\label{full-risk-eq}
	R_{\ell} \br{\theta,A} ~:=~ \Ee{\textsf{O}_1^n \sim P_\theta^n}{\Ee{a \sim A(\textsf{O}_1^n)}{\ell \br{\theta,a}}}.
\end{equation}
Similarly we can define the risk in terms of regret as follows: 
\begin{equation}
	\label{full-risk-in-regret-eq}
	R_{\Delta \ell} \br{\theta,A} ~:=~ \Ee{\textsf{O}_1^n \sim P_\theta^n}{\Ee{a \sim A(\textsf{O}_1^n)}{\Delta \ell \br{\theta,a}}}.
\end{equation}

For any fixed (unknown) parameter $\theta \in \Theta$, the goal is to find an optimal stochastic decision rule. Two main approaches to achieve this goal are:
\begin{itemize}
	\item \textit{Bayesian approach} (average case analysis), which is more appropriate if the decision maker has some intuition about $\theta$, given in the form of a prior probability distribution $\pi$, and
	\item \textit{Minimax approach} (worst case analysis), which is more appropriate if the decision maker has no prior knowledge concerning $\theta$. In this paper, we focus on this strategy. 
\end{itemize}

We measure the difficulty of the general learning task by the \textit{minimax risk} defined as,
\begin{align}
	\label{minimax-risk-eq}
	\underline{R}_{\ell}^\star ~:=~ \inf_{A}{\sup_{\theta \in \Theta}{R_{\ell} \br{\theta,A}}}
\end{align}
By replacing $\ell$ by $\Delta \ell$ in \eqref{minimax-risk-eq}, we obtain $\underline{R}_{\Delta \ell}^\star$. Below we discuss some specific instantiations (supervised learning, binary classification, and parameter estimation) of this general learning task.

\subsection{Supervised Learning Problem}
\label{subsec:supervised-learning-task}
Let $\mathcal{X} \times \mathcal{Y}$ be a measurable space, and let $D$ be an unknown joint probability measure on $\mathcal{X} \times \mathcal{Y}$. The set $\mathcal{X}$ is called the \textit{instance space}, the set $\mathcal{Y}$ the \textit{outcome space}. Let $\textsf{S} = \bc{\br{\textsf{X}_i , \textsf{Y}_i}}_{i=1}^n \in \br{\mathcal{X} \times \mathcal{Y}}^n$ be a finite training sample, where each pair $\br{\textsf{X}_i , \textsf{Y}_i}$ is generated independently according to the unknown probability measure $D$. Then the goal of a learning algorithm is to find a function $f:\mathcal{X} \rightarrow \mathcal{Y}$ which given a new instance $x \in \mathcal{X}$, predicts its label to be $\hat{y} = f\br{x}$.

In order to measure the performance of a learning algorithm, we define an \textit{error function} $d:\mathcal{Y} \times \mathcal{Y} \rightarrow \mathbb{R}_+$,
where $d\br{y,\hat{y}}$ quantifies the discrepancy between the predicted value $\hat{y}$ and the actual value $y$. The performance of any function $f:\mathcal{X} \rightarrow \mathcal{Y}$ is then measured in terms of its \textit{generalization error}, which is defined as the expected error:
\begin{equation}
\label{generalization-error-eq}
\mathrm{er}_d \br{f,D} ~:=~ \Ee{\br{\textsf{X},\textsf{Y}} \sim D}{d\br{\textsf{Y},f\br{\textsf{X}}}} ,
\end{equation}
where the expectation is taken with respect to the probability measure $D$ on the data $\br{\textsf{X},\textsf{Y}}$. The best estimate $f_D^\star \in \mathcal{Y}^{\mathcal{X}}$ is therefore the one for which the generalization error is as small as possible, that is,
\begin{equation}
\label{optimal-hypothesis-all}
f_D^\star ~:=~ \argmin_{f \in \mathcal{Y}^{\mathcal{X}}}{\mathrm{er}_d \br{f,D}} .
\end{equation}
The function $f_D^\star$ is called the target hypothesis. Given a fixed hypothesis class $\mathcal{F} \subseteq \mathcal{Y}^\mathcal{X}$, the goal of a learning algorithm is thus to choose the hypothesis function $f^\star \in \mathcal{F}$ which has the smallest generalization error on data drawn according to the underlying probability measure $D$,
\begin{equation}
\label{optimal-hypothesis-cF}
f_{D,\mathcal{F}}^\star ~:=~ \argmin_{f \in \mathcal{F}}{\mathrm{er}_d \br{f,D}} .
\end{equation}
We will assume in the following that such an $f_{D,\mathcal{F}}^\star$ exists.

The supervised learning problem can be derived from the general learning task with the following instantiation: 
\begin{itemize}
	\item the observation space is $\mathcal{O} = \mathcal{X} \times \mathcal{Y}$, where $\mathcal{X} \subseteq \mathbb{R}^d$,
	\item the action space is $\mathcal{A} = \mathcal{F} \subseteq \mathcal{Y}^\mathcal{X}$, 
	\item the learning algorithm is $A = \hat{f}$, and
	\item the loss function is 
	\[
	\ell_d : \Theta \times \mathcal{F} \ni \br{\theta,f} \mapsto \ell_d \br{\theta,f} := \mathrm{er}_d \br{f,P_\theta} \in \mathbb{R}_+ ,
	\] 
	where $P_\theta$ is the probability measure associated with the parameter $\theta \in \Theta$. One needs to carefully distinguish between the error function $d: \mathcal{Y} \times \mathcal{Y} \rightarrow \mathbb{R}$ which acts on the observation space, and the loss function $\ell_d: \Theta \times \mathcal{A} \rightarrow \mathbb{R}$ which acts on the parameter and decision spaces.
\end{itemize}

\subsection{Binary Classification}
\label{subsec:binary-classification}
When $\mathcal{Y} = \bc{-1,1}$, the supervised learning task is called binary classification, which is a central problem in machine learning (\cite{devroye2013probabilistic}). A common error function for binary classification is simply the zero-one error defined by $d_{0-1} \br{y,\hat{y}} = \ind{\hat{y} \neq y}$. In this case the generalization error of a classifier $f:\mathcal{X} \rightarrow \bc{-1,1}$ w.r.t.\ a probability measure $D$ is simply the probability that it predicts the wrong label on a randomly drawn example:
\begin{align*}
\mathrm{er}_{d_{0-1}} \br{f,D} ~:=~& \Ee{\br{\textsf{X},\textsf{Y}} \sim D}{d_{0-1}\br{\textsf{Y},f\br{\textsf{X}}}} \\
~=~& \Pp{\br{\textsf{X},\textsf{Y}}\sim D}{f\br{\textsf{X}} \neq \textsf{Y}} .
\end{align*}
The optimal error over all possible classifiers $f:\mathcal{X} \rightarrow \bc{-1,1}$ for a given probability measure $D$ is called the \textit{Bayes error} (minimum generalization error) associated with $D$:
\begin{equation}
\label{bayes-error-eq}
\underline{\mathrm{er}}_{d_{0-1}} \br{D} ~:=~  \inf_{f \in \bc{-1,1}^\mathcal{X}}{\mathrm{er}_{d_{0-1}} \br{f,D}} .
\end{equation}
It is easily verified that, if $\eta_D\br{x}$ is defined as the conditional probability (under $D$) of a positive label given
$x$, $\eta_D\br{x} = \mathbb{P}_D\bs{\textsf{Y}=1 \mid \textsf{X}=x}$, then the classifier $f_D^\star:\mathcal{X} \rightarrow \bc{-1,1}$ given by
\[
f_D^\star \br{x} = \begin{cases}
1 & \text{if } \eta_D\br{x} \geq 1/2 \\
-1 & \text{otherwise}
\end{cases}
\]
achieves the Bayes error. Such a classifier is termed a \textit{Bayes classifier}. In general, $\eta_D$ is unknown so the above classifier cannot be constructed directly. 

By defining $\ell_{d_{0-1}} : \Theta \times \mathcal{F} \ni \br{\theta , f} \mapsto \ell_{d_{0-1}} \br{\theta , f} := \mathrm{er}_{d_{0-1}} \br{f,P_\theta} \in \mathbb{R}_+$, the binary classification problem can be compactly represented by the tuple $\br{\Theta,\br{\mathcal{X} \times \bc{-1,1}}^n,\mathcal{F},\ell_{d_{0-1}}}$. Using the Bayes rule, the distribution $\mathbb{P}_{\theta}$ (which is a short hand for $\mathbb{P}_{P_\theta}$) can be decomposed as follows: 
\begin{align*}
\mathbb{P}_{\theta} \bs{\textsf{X}=x,\textsf{Y}=1} ~=~& \mathbb{P}_{\theta} \bs{\textsf{X}=x} \cdot \mathbb{P}_{\theta} \bs{\textsf{Y}=1 \mid \textsf{X}=x} \\
~=~& M_{\theta} \br{x} \cdot \eta_{\theta} \br{x} , 
\end{align*} 
where $M_{\theta} \br{x} := \mathbb{P}_{\theta} \bs{\textsf{X}=x}$ and $\eta_{\theta} \br{x} := \mathbb{P}_{\theta} \bs{\textsf{Y}=1 \mid \textsf{X}=x}$.  

\subsection{Cost-sensitive Binary Classification}
\label{subsec:cost-classification}
Suppose we are given gene expression profiles for some number of patients, together with labels for these patients indicating whether or not they had a certain form of a disease. We want to design a learning algorithm which automatically recognizes the diseased patient based on the gene expression profile of a patient. In this case, there are different costs associated with different types of mistakes (the health risk for a false label ``no'' is much higher than for a false ``yes''), and the \textit{cost-sensitive error function}  (for $c \in (0,1)$) can be used to capture this:
\[
d_c \br{y,\hat{y}} := \ind{\hat{y} \neq y} \cdot \bc{\bar{c} \cdot \ind{y=1} + c \cdot \ind{y=-1}} ,
\]
where $\bar{c}:= 1-c$. Then the performance measure (loss function) associated with the above cost-sensitive error function is given by 
\[
\ell_{d_{c}} : \Theta \times \mathcal{F} \ni \br{\theta , f} \mapsto \ell_{d_{c}} \br{\theta , f} := \mathrm{er}_{d_{c}} \br{f,P_\theta} \in \mathbb{R}_+ , 
\] 
where $\mathrm{er}_{d_{c}} \br{f,P_\theta}$ is given by 
\[
\Ee{\br{\textsf{X},\textsf{Y}} \sim P_\theta}{\ind{f\br{\textsf{X}} \neq \textsf{Y}} \cdot \bc{\bar{c} \cdot \ind{\textsf{Y}=1} + c \cdot \ind{\textsf{Y}=-1}}} .
\]

For any $\eta:\mathcal{X} \rightarrow \bs{0,1}$, and $f:\mathcal{X} \rightarrow \bc{-1,1}$, define the conditional generalization error (given $x \in \mathcal{X}$) as 
\begin{align*}
\mathrm{er}_{d_{c}} \br{f,\eta ; x} ~:=~& \Ee{\textsf{Y} \sim \eta\br{x}}{d_{c}\br{\textsf{Y},f\br{x}}} \\
~=~& \bar{c} \cdot \eta\br{x} \cdot \ind{f\br{x} \neq 1} \\
& + c \cdot \bar{\eta\br{x}} \cdot \ind{f\br{x} \neq -1} , 
\end{align*}
where $\bar{\eta\br{x}} := 1- \eta\br{x}$. Then $\mathrm{er}_{d_{c}} \br{f,\eta ; x}$ is minimized by 
\begin{align*}
f^\star\br{x} ~:=~& \argmin_{f \in \bc{-1,1}^\mathcal{X}} {\Ee{\textsf{Y} \sim \eta (x)}{d_{c}\br{\textsf{Y},f\br{x}}}} \\
~=~& \mathrm{sign}\br{\bar{c} \cdot \eta\br{x} - c \cdot \bar{\eta\br{x}}} \\
~=~& \mathrm{sign}\br{\eta\br{x} - c} ,
\end{align*}
since $\mathrm{er}_{d_{c}} \br{f^\star,\eta ; x} = \bar{c} \cdot \eta\br{x} \wedge c \cdot \bar{\eta\br{x}}$. In order to find the optimal classifier for each $\theta \in \Theta$ (associated joint probability measure $P_\theta$ on $\mathcal{X} \times \bc{-1,1}$) w.r.t.\ the cost-sensitive loss function, we note that
\begin{align*}
& \inf_{f \in \bc{-1,1}^\mathcal{X}} \ell_{d_{c}} \br{\theta,f} \\
~=~&
\inf_{f \in \bc{-1,1}^\mathcal{X}} \mathrm{er}_{d_{c}} \br{f,P_\theta} \\ 
~=~& \inf_{f \in \bc{-1,1}^\mathcal{X}} \Ee{\textsf{X} \sim M_\theta}{\Ee{\textsf{Y} \sim \eta_\theta \br{\textsf{X}}}{d_{c}\br{\textsf{Y},f\br{\textsf{X}}}}} \\
~=~& \Ee{\textsf{X} \sim M_\theta}{\inf_{f \in \bc{-1,1}^\mathcal{X}}{\Ee{\textsf{Y} \sim \eta_\theta \br{\textsf{X}}}{d_{c}\br{\textsf{Y},f\br{\textsf{X}}}}}} \\
~=~& \ell_{d_{c}} \br{\theta,f_\theta^\star}, 
\end{align*}
where $M_{\theta} \br{x} := \mathbb{P}_{\theta} \bs{\textsf{X}=x}$, $\eta_{\theta} \br{x} := \mathbb{P}_{\theta} \bs{\textsf{Y}=1 \mid \textsf{X}=x}$, and $f_\theta^\star$ is given by
\begin{equation}
\label{bayes-classifier-lc}
f_\theta^\star\br{x} := \begin{cases}
1, & \text{if } \eta_\theta(x) \geq c \\
-1, & \text{otherwise}
\end{cases} .
\end{equation}

We instantiate the regret, risk and minimax risk of the cost-sensitive classification problem as follows
\begin{align*}
\Delta \ell_{d_{c}} \br{\theta,f} :=& \ell_{d_{c}} \br{\theta,f} - \ell_{d_{c}} \br{\theta,f_\theta^\star} \\
\mathcal{R}_{\Delta \ell_{d_{c}}} \br{\theta,\hat{f}} :=& \Ee{\textsf{S} \sim P_\theta^n}{\Ee{f \sim \hat{f}\br{\textsf{S}}}{\Delta \ell_{d_c} \br{\theta,f}}} \\
\underline{\mathcal{R}}_{\Delta \ell_{d_{c}}}^\star :=& \inf_{\hat{f}}{\sup_{\theta \in \Theta}{\mathcal{R}_{\Delta \ell_{d_{c}}} \br{\theta,\hat{f}}}} ,
\end{align*}
respectively, where $\textsf{S} = \bc{\br{\textsf{X}_i , \textsf{Y}_i}}_{i=1}^n$. The following lemma from \cite{scott2012calibrated} will be used later.
\begin{lemma}[\cite{scott2012calibrated}]
	\label{cost-sensitive-lemma}
	Consider the binary classification problem $\br{\Theta,\br{\mathcal{X} \times \bc{-1,1}}^n,\mathcal{F},\ell_{d_{c}}}$. For any $f \in \mathcal{F}$ and $c \in (0,1)$,
	\[
	\Delta \ell_{d_{c}} \br{\theta,f} ~=~ \frac{1}{2} \cdot \Ee{\textnormal{\textsf{X}}\sim M_\theta}{\abs{\eta_\theta \br{\textnormal{\textsf{X}}} - c} \cdot \abs{f\br{\textnormal{\textsf{X}}} - f_\theta^\star \br{\textnormal{\textsf{X}}}}} ,
	\]
	where $f_\theta^\star$ is given by \eqref{bayes-classifier-lc}.
\end{lemma}

\subsection{Parameter Estimation Problem}
\label{subsec:parameter-estimation}
The main goal of a parameter estimation problem is to accurately \textit{reconstruct the parameters} (with $\mathcal{A} = \Theta$, and $A = \hat{\theta}$) of the original distribution from which the data is generated, using the loss function of the type $\rho:\Theta\times\Theta\rightarrow\mathbb{R}$ (which satisfies symmetry and the triangle inequality). This problem is represented by the tuple $\br{\Theta,\mathcal{O}^n,\Theta,\rho}$. The minimax risk of this problem is defined as 
\begin{equation}
\label{minimax-risk-parameter}
\underline{\mathcal{R}}_\rho^\star ~:=~ \inf_{\hat{\theta}} \sup_{\theta \in \Theta} \Ee{\textsf{O}_1^n \sim P_{\theta}^n}{\Ee{\tilde{\theta} \sim \hat{\theta}\br{\textsf{O}_1^n}}{\rho\br{\theta,\tilde{\theta}}}} .
\end{equation}

\subsection{$f$-Divergences}
\label{subsec:f-divergences}
The hardness of the binary classification problem depends on the distinguishability of the two probability distributions associated with it. The class of $f$-divergences (\cite{ali1966general,csiszar1972class}) provide a rich set of relations that can be used to measure the separation of the distributions in a binary experiment.
\begin{definition} 
	\label{binary-fdiv-def}
	Let $f: \br{0,\infty} \rightarrow \mathbb{R}$ be a convex function with $f(1) = 0$. For all distributions $P, Q \in \mathcal{P}\br{\mathcal{O}}$ the $f$-divergence between $P$ and $Q$ is,
	\[
	\mathbb{I}_f \br{P,Q} = \Ee{Q}{f\br{\frac{dP}{dQ}}} = \int_{\mathcal{O}}^{} {f\br{\frac{dP}{dQ}} dQ}
	\]
	when $P$ is absolutely continuous with respect to $Q$ and equals $\infty$ otherwise.
\end{definition}
Many commonly used divergences in probability, mathematical statistics and information theory are special cases of $f$-divergences. For example: 
\begin{enumerate}
	\item The Kullback-Leibler divergence (with $f_{\mathrm{KL}} \br{u} = u \log{u}$)
	\[
	\mathbb{I}_{\mathrm{KL}}\br{P , Q} = D\br{P \mid\mid Q} = \Ee{Q}{\frac{dP}{dQ} \log{\frac{dP}{dQ}}}
	\] 
	\item The total variation distance (with $f_{\mathrm{TV}} \br{u} = \abs{u-1}$) 
	\[
	\mathbb{I}_{\mathrm{TV}}\br{P , Q} = d_{\mathrm{TV}}\br{P,Q} = \Ee{Q}{\abs{\frac{dP}{dQ} - 1}} . 
	\]
	Also for general measures $\mu$ and $\nu$ on $\mathcal{O}$, we define $d_{\mathrm{TV}}\br{\mu,\nu} = \int{\abs{d\mu - d\nu}}$. 
	\item The $\chi^2$-divergence (with $f_{\chi^2} \br{u} = (u-1)^2$) 
	\[
	\mathbb{I}_{\chi^2}\br{P , Q} = \chi^2 \br{P \mid\mid Q} = \Ee{Q}{\br{\frac{dP}{dQ} - 1}^2}
	\]
	\item The squared Hellinger distance (with $f_{\mathrm{He}^2} \br{u} = (\sqrt{u}-1)^2$)
	\[
	\mathbb{I}_{\mathrm{He}^2}\br{P , Q} = \mathrm{He}^2 \br{P , Q} = \Ee{Q}{\br{\sqrt{\frac{dP}{dQ}} - 1}^2}
	\]
\end{enumerate}

\paragraph{Integral Representations of $f$-divergences: }
Representation of $f$-divergences and loss functions as weighted average of \textit{primitive} components (in the sense that they can be used to express other measures but themselves cannot be so expressed) is very useful in studying certain geometric properties of them using the weight function behavior. The following restatement of a theorem by \cite{liese2006divergences} provides such a representation for any $f$-divergence (confer \cite{reid2011information} for a proof):
\begin{theorem}
	\label{fdiv-weight-theorem}
	Define $\bar{c} := 1-c$, for $c \in \bs{0,1}$, and let $f$ be convex such that $f(1) = 0$. Then the $f$-divergence between $P$ and $Q$ can be written in a weighted integral form as follows:
	\begin{equation}
	\label{weighted-int-rep-fdiv}
	\mathbb{I}_f \br{P,Q} = \int_{0}^{1}{\mathbb{I}_{f_c}\br{P,Q} \gamma_f\br{c} dc} , 
	\end{equation}
	where 
	\begin{equation}
	\label{primitive-fdiv-tent-func}
	f_c (t) = \bar{c} \wedge c - \bar{c} \wedge (c t), \, t \in \mathbb{R}_+
	\end{equation}
	and 
	\begin{equation}
	\label{fdiv-weight}
	\gamma_f\br{c} := \frac{1}{c^3} f'' \br{\frac{\bar{c}}{c}} .
	\end{equation}
\end{theorem}
For $c \in \bs{0,1}$, the term $\mathbb{I}_{f_c} (P,Q)$ in \eqref{weighted-int-rep-fdiv} is called the \textit{$c$-primitive $f$-divergence} and can be written as 
\begin{align}
\mathbb{I}_{f_c} (P,Q) ~=~& \int{\bc{\bar{c} \wedge c - \bar{c} \wedge \br{c \frac{dP}{dQ}}} dQ} \label{eq-primitive-f-div-1}\\
~=~& \bar{c} \wedge c - \int{\bar{c} dQ \wedge c dP} \label{eq-primitive-f-div-2} \\
~=~& \bar{c} \wedge c - \frac{1}{2} + \frac{1}{2} \int{\abs{c dP - \bar{c} dQ}} \label{eq-primitive-f-div-3} \\
~=~& \frac{1}{2} d_{\mathrm{TV}} \br{c P , \bar{c} Q} - \frac{1}{2} \abs{1-2 c} , \label{eq-primitive-f-div-4}
\end{align}
where the first equality \eqref{eq-primitive-f-div-1} is due to the definition of $f$-divergence and \eqref{primitive-fdiv-tent-func}, and the third equality \eqref{eq-primitive-f-div-3} is due to the following observation:
\begin{align*}
\int{\abs{p-q}} ~=~& \int_{q \geq p}{q-p} + \int_{q < p}{p-q} \\
~=~& \int_{q \geq p}{q} + \int_{q < p}{p} - \int{p \wedge q} \\
~=~& 1 - \int_{q < p}{q} + 1 - \int_{q \geq p}{p} - \int{p \wedge q} \\
~=~& 2 - 2 \int{p \wedge q} . \\
\end{align*}

\paragraph{Comparison between $f$-Divergences: }
Consider the problem of maximizing or minimizing an $f$-divergence between two probability measures subject to a constraint on another $f$-divergence. This problem is captured by the following definition: 
\begin{definition}[Joint Range]
	Consider two $f$-divergences $F_{P,Q} := \mathbb{I}_f \br{P,Q}$ and $G_{P,Q} := \mathbb{I}_g \br{P,Q}$. Their joint
	range is a subset of $\mathbb{R}^2$ defined by
	\begin{align*}
	J ~:=~& \bc{\br{F_{P,Q}, G_{P,Q}}: P,Q \in \mathcal{P}\br{\mathcal{O}}, \mathcal{O} \text{ is measurable}} , \\
	{J}_k ~:=~& \bc{\br{F_{P,Q}, G_{P,Q}}: P,Q \in \mathcal{P}\br{[k]}}, \, k \in \mathbb{N} .
	\end{align*}
\end{definition}
The region ${J}$ seems difficult to characterize since we need to consider $P,Q$ over all measurable spaces; on the other hand, the region ${J}_k$ for small $k$ is easy to obtain. The following theorem relates these two regions (${J}$ and ${J}_k$). 
\begin{theorem}[\cite{harremoes2011pairs}]
	\label{joint-range-theo}
	${J} = \text{conv}\br{{J}_2} .$
\end{theorem}
By Theorem~\ref{joint-range-theo}, the region $J$ is no more than the convex hull of ${J}_2$. In certain cases, it is easy to obtain a parametric formula of ${J}_2$. In those cases, we can systematically prove several important inequalities between two $f$-divergences via their joint range. For example using the joint range between the total variation and Hellinger divergence, it can be shown that (\cite{Tsybakov2009nonparbook,yurilecnotes}):
\begin{equation}
\label{joint-range-hell-var}
d_\mathrm{TV}\br{P,Q} ~\leq~ \mathrm{He}\br{P,Q} \sqrt{1 - \frac{\mathrm{He}^2\br{P,Q}}{4}} .
\end{equation}
We extend the above result to the $c$-primitive $f$-divergence as follows: 
\begin{equation}
\label{joint-range-hell-primc}
\mathbb{I}_{f_c}\br{P,Q} ~\leq~ \br{c \wedge \bar{c}} \cdot \mathrm{He}\br{P,Q} \sqrt{1 - \frac{\mathrm{He}^2\br{P,Q}}{4}} .
\end{equation}
We use a mathematical software to plot (see Figure~\ref{joint-hell-fc}) the joint range between the $c$-primitive $f$-divergence and the Hellinger divergence which is given by the convex hull of
\[
J_2 := \bc{\br{F_{p,q},G_{p,q}} : p,q \in \bs{0,1}} ,
\]
where $F_{p,q} = 2 \br{1-\sqrt{p q} -\sqrt{\bar{p}\bar{q}}}$, and $G_{p,q} = \frac{1}{2}\br{\abs{c p - \bar{c} q} + \abs{c \bar{p} - \bar{c} \bar{q}} - \abs{2 c - 1}}$. Then using this joint range, we verify that the bound given in \eqref{joint-range-hell-primc} is indeed true. 

\begin{figure}
	\centering
	\includegraphics[height=0.8\linewidth]{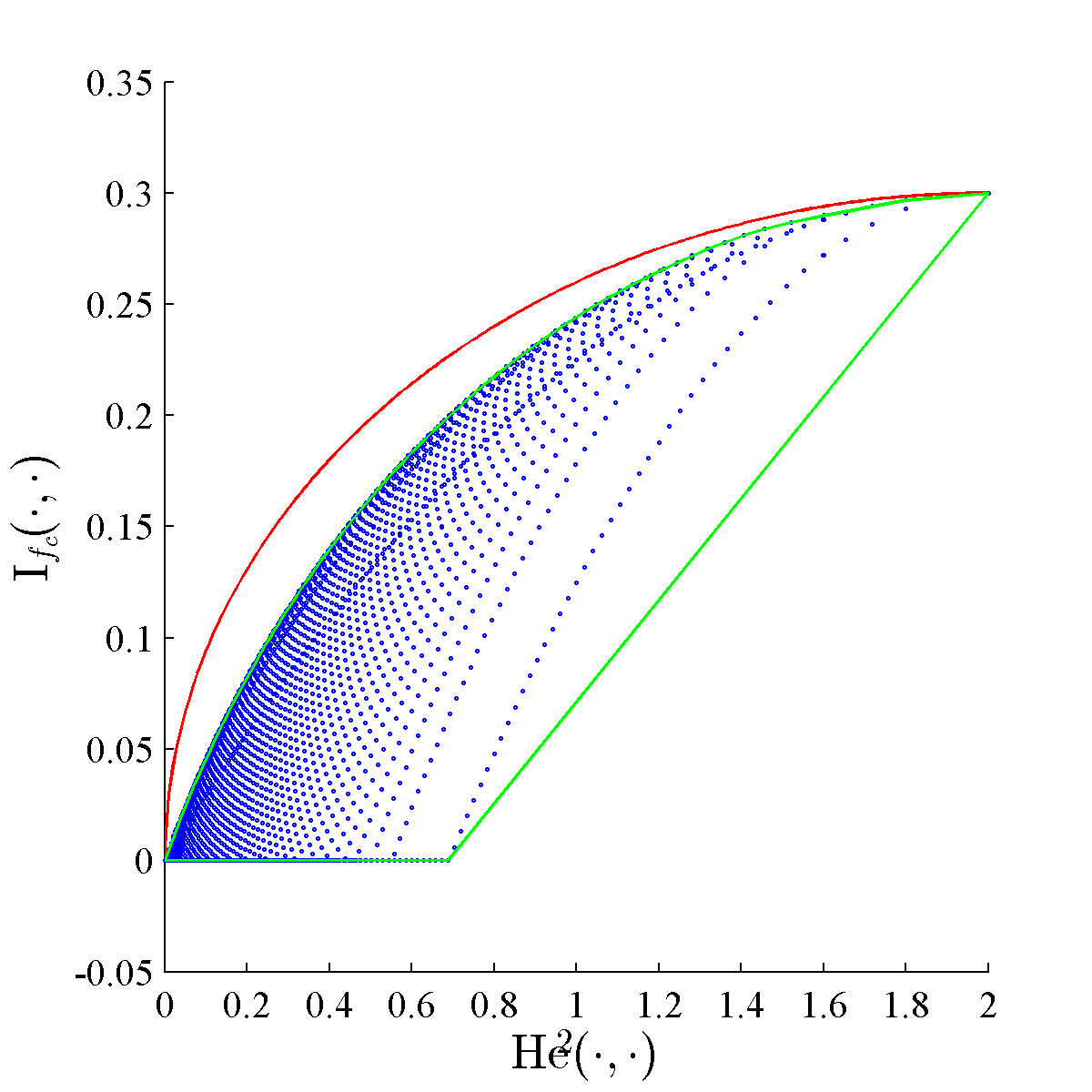}
	\caption[Joint range of Hellinger distance and a $c$-primitive $f$-divergence]{Joint range ($J_2$) of Hellinger distance and a $c$-primitive $f$-divergence (c=0.7) (\textcolor{blue}{$\cdots$}), convex hull of $J_2$ (\textcolor{green}{---}), and a parametric curve $\bc{\br{\mathrm{He}^2 \br{P,Q},\br{c \wedge \bar{c}} \cdot \mathrm{He} \br{P,Q} \sqrt{1 - \frac{\mathrm{He}^2 \br{P,Q}}{4}}}: \mathrm{He}^2 \br{P,Q} \in \bs{0,2}}$ (\textcolor{red}{---}). \label{joint-hell-fc}}
\end{figure}

\paragraph{Sub-additive $f$-Divergences: }
Some $f$-divergences satisfy the sub-additivity property, which will be useful in analyzing the hardness of learning problems with repeated experiments (samples). The following lemma shows that both total variation and squared Hellinger divergences satisfy this property.  
\begin{lemma}
	\label{sub-add-f-div}
	For all collections of distributions $P_i , Q_i \in \mathcal{P} \br{\mathcal{O}_i}$, $i \in [k]$
	\[
	d_{\mathrm{TV}} \br{\bigotimes_{i=1}^k P_i , \bigotimes_{i=1}^k Q_i} ~\leq~ \sum_{i=1}^{k}{d_{\mathrm{TV}} \br{P_i ,Q_i}} ,
	\]	
	and 
	\[
	\mathrm{He}^2 \br{\bigotimes_{i=1}^k P_i , \bigotimes_{i=1}^k Q_i} ~\leq~ \sum_{i=1}^{k}{\mathrm{He}^2 \br{P_i ,Q_i}} .
	\]
\end{lemma}

	\section{Hardness of the Cost-sensitive Classification Problem}
\label{sec:hard-cost-sense}
In this section we follow the presentation of \cite{raginskylecnotes}. Before studying the hardness of the cost-sensitive classification, we study the hardness of the auxiliary problem of parameter estimation.

\subsection{Minimax Lower Bounds for Parameter Estimation Problem}
\label{subsec:minmax-paramter-estimation}
We derive the cost-dependent lower bound for $\underline{\mathcal{R}}_\rho^\star$ (defined in \eqref{minimax-risk-parameter}) by extending the standard Le Cam and Assouad's techniques. We start with the two point method introduced by Lucien Le Cam for obtaining minimax lower bounds. 
\begin{proposition}
	\label{lecam-method}
	For any $c \in \br{0,1}$, the minimax risk $\underline{\mathcal{R}}_\rho^\star$ (given by \eqref{minimax-risk-parameter}) of the parameter estimation problem with (pseudo metric) loss function $\rho:\Theta \times \Theta \rightarrow \mathbb{R}$ is bounded from below as follows:
	\[
	\underline{\mathcal{R}}_\rho^\star ~\geq~ \sup_{\theta \neq \theta'}{\bc{\rho\br{\theta,\theta'} \cdot \br{c \wedge \bar{c} - \mathbb{I}_{f_c}\br{P_{\theta}^n,P_{\theta'}^n}}}} ,
	\]
	where $f_c$ is given by \eqref{primitive-fdiv-tent-func}.
\end{proposition}

By setting $c = \frac{1}{2}$ in Proposition~\ref{lecam-method}, we recover Le Cam's (\cite{le2012asymptotic}) minimax lower bound for parameter estimation problem: 
\[
\underline{\mathcal{R}}_\rho^\star ~\geq~ \frac{1}{2} \sup_{\theta \neq \theta'}{\bc{\rho\br{\theta,\theta'} \cdot \br{1 - \frac{1}{2} d_{\mathrm{TV}}\br{P_{\theta}^n,P_{\theta'}^n}}}} ,
\]
since $\mathbb{I}_{f_{\frac{1}{2}}} \br{P,Q} = \frac{1}{4} d_{\mathrm{TV}}\br{P,Q}$ (from \eqref{eq-primitive-f-div-3} with $c=\frac{1}{2}$). Now we provide an auxiliary result which will be useful in deriving the cost-dependent minimax lower bounds via Assouad's lemma (\cite{assouad1983deux}). 
\begin{corollary}
	\label{auxi-minimax-coro}
	Let $\pi$ be any prior distribution on $\Theta$, and let $\mu$ be any joint probability distribution of a random pair $(\theta,\theta')\in\Theta \times \Theta$, such that the marginal distributions of both $\theta$ and $\theta'$ are equal to $\pi$. Then for any $c \in \br{0,1}$, the minimax risk $\underline{\mathcal{R}}_\rho^\star$ (given by \eqref{minimax-risk-parameter}) of the parameter estimation problem is bounded from below as follows:
	\[
	\underline{\mathcal{R}}_\rho^\star ~\geq~ \Ee{(\theta,\theta') \sim \mu}{\rho\br{\theta,\theta'} \cdot \br{c \wedge \bar{c} - \mathbb{I}_{f_c}\br{P_{\theta}^n,P_{\theta'}^n}}} .
	\]
\end{corollary} 

Using the above corollary and extending the standard Assouad's lemma, we derive the cost-dependent minimax lower bound for the parameter estimation problem. 
\begin{theorem}
	\label{assouad-method}
	Let $d \in \mathbb{N}$, $\Theta=\bc{-1,1}^d$ and $\rho = \rho_{\mathrm{Ha}}$, where the Hamming distance $\rho_{\mathrm{Ha}}$ is given by \eqref{eq-hamming-distance}. Then for any $c \in \br{0,1}$, the minimax risk of the parameter estimation problem satisfies
	\[
	\underline{\mathcal{R}}_{\rho_{\mathrm{Ha}}}^\star ~\geq~ d \br{c \wedge \bar{c} - \underset{\theta,\theta': \rho_{\mathrm{Ha}}\br{\theta,\theta'}=1}{\max} \mathbb{I}_{f_c}\br{P_{\theta}^n,P_{\theta'}^n}} .
	\]
\end{theorem}
By setting $c = \frac{1}{2}$ in Theorem~\ref{assouad-method}, we recover the standard Assouad's lemma (\cite{assouad1983deux}):
\[
\underline{\mathcal{R}}_{\rho_{\mathrm{Ha}}}^\star ~\geq~ \frac{d}{2} \br{1 - \frac{1}{2} \underset{\theta,\theta': \rho_{\mathrm{Ha}}\br{\theta,\theta'}=1}{\max} d_{\mathrm{TV}}\br{P_{\theta}^n,P_{\theta'}^n}} .
\]

We use the following two properties of the Hellinger distance $\mathrm{He}^2 \br{P,Q}$ to derive a more practically useful version of Assouad's lemma: 
\begin{itemize}
	\item $\mathbb{I}_{f_c}\br{P,Q} ~\leq~ \br{c \wedge \bar{c}} \cdot \mathrm{He} \br{P,Q}$, for all distributions $P, Q \in \mathcal{P}\br{\mathcal{O}}$ (refer \eqref{joint-range-hell-primc})	
	\item $\mathrm{He}^2 \br{\bigotimes_{i=1}^k P_i , \bigotimes_{i=1}^k Q_i} ~\leq~ \sum_{i=1}^{k}{\mathrm{He}^2 \br{P_i ,Q_i}}$, for all distributions $P_i, Q_i \in \mathcal{P}\br{\mathcal{O}_i}$, $i \in [k]$	
\end{itemize}
Armed with these facts, we prove the following version of Assouad's lemma: 
\begin{corollary}
	\label{assouad-corro}
	Let $\mathcal{O}$ be some set and $c \in [0,1]$. Define 
	\[
	\mathcal{P} \br{\mathcal{O}} := \bc{P_{\theta} \in \mathcal{P} \br{\mathcal{O}} : \theta \in \bc{-1,1}^d}
	\]
	be a class of probability measures induced by the parameter space $\Theta = \bc{-1,1}^d$. Suppose that there exists some function $\alpha: \bs{0,1} \rightarrow \mathbb{R}_+$, such that 
	\[
	\mathrm{He}^2 \br{P_{\theta},P_{\theta'}} \leq \alpha\br{c} , \quad \text{ if } \rho_{\mathrm{Ha}}\br{\theta,\theta'} = 1 ,
	\] 
	i.e. the two probability distributions $P_{\theta}$ and $P_{\theta'}$ (associated with the two parameters $\theta$ and $\theta'$ which differ  only in one coordinate) are sufficiently close w.r.t.\ Hellinger distance. Then the minimax risk of the parameter estimation problem with parameter space $\Theta = \bc{-1,1}^d$ and the loss function $\rho = \rho_{\mathrm{Ha}}$ is bounded below by
	\begin{equation}
	\label{assouad-practical-eq}
	\underline{\mathcal{R}}_{\rho_{\mathrm{Ha}}}^\star ~\geq~ d \cdot (c \wedge \bar{c}) \cdot \br{1 - \sqrt{\alpha\br{c} n}} .
	\end{equation}
\end{corollary}

The number of training samples $n$ appear in the minimax lower bound \eqref{assouad-practical-eq}. Thus the hardness of the problem can be expressed as a function of the sample size along with other problem specific parameters.

\subsection{Minimax Lower Bounds for Cost-sensitive Classification Problem}
\label{subsec:minmax-cost-sensitive-classification}
A natural question to ask regarding cost-sensitive classification problem is how does the hardness of the problem depend upon the cost parameter $c \in \bs{0,1}$. Let $\mathcal{F} \subseteq \bc{-1,1}^\mathcal{X}$ be the action space and $h \in \bs{0,c \wedge \bar{c}}$ be the margin parameter whose interpretation is explained below. Then we choose a parameter space $\Theta_{h,\mathcal{F}}$ such that:
\begin{enumerate}
	\item $\forall{\theta \in \Theta_{h,\mathcal{F}}}$, $f_\theta^\star \in \mathcal{F}$, where $f_\theta^\star$ is given by \eqref{bayes-classifier-lc}. That is we restrict the parameter space s.t. the Bayes classifier associated with each choice of parameter lies within the predetermined function class $\mathcal{F}$. 
	\item \begin{equation}
	\label{noise-condition}
	\abs{\eta_\theta \br{\textsf{X}} - c} \geq h \text{ a.s. } \forall{\theta \in \Theta_{h,\mathcal{F}}} .
	\end{equation}
	This condition is a generalized notion of \textit{Massart noise condition} with margin $h \in \bs{0 , c \wedge \bar{c}}$ \cite{massart2006risk}. The motivation for this condition is well established by \cite{massart2006risk}. They have argued that under certain ``margin'' type conditions \cite{tsybakov2004optimal} like this, it is possible to design learning algorithms for the binary classification problem, with better rates compared to the case where no such condition is satisfied. 
\end{enumerate}
Thus we consider the problem represented by $\br{\Theta_{h,\mathcal{F}}, \br{\mathcal{X} \times \bc{-1,1}}^n, \mathcal{F}, \ell_{d_{c}}}$ and the minimax risk $\underline{\mathcal{R}}_{\Delta \ell_{d_{c}}}^\star$ (in terms of regret) of it given by 
\begin{equation}
\label{cost-minmax-risk-eq}
\inf_{\hat{f}}{\sup_{\theta \in \Theta_{h,\mathcal{F}}}{\Ee{\bc{\br{\textsf{X}_i,\textsf{Y}_i}}_{i=1}^n \sim P_{\theta}^n}{\Ee{f \sim \hat{f}\br{\bc{\br{\textsf{X}_i,\textsf{Y}_i}}_{i=1}^n}}{\Delta \ell_{d_c} \br{\theta,f}}}}} .
\end{equation}
The following is a generalization of the result proved in \cite[Theorem~4]{massart2006risk} for $c = \frac{1}{2}$.
\begin{theorem}
	\label{main-theo-cost-classi}
	Let $\mathcal{F}$ be a VC class of binary-valued functions on $\mathcal{X}$ with VC dimension (refer Appendix~\ref{sec:vc-dimension-sec}) $V \geq 2$. Then for any $n \geq V$ and any $h \in [0,c \wedge \bar{c}]$, the minimax risk \eqref{cost-minmax-risk-eq} of the cost-sensitive binary classification problem $\br{\Theta_{h,\mathcal{F}}, \br{\mathcal{X} \times \bc{-1,1}}^n, \mathcal{F}, \ell_{d_{c}}}$ is lower bounded as follows: 
	\[
	\underline{\mathcal{R}}_{\Delta \ell_{d_c}}^\star ~\geq~ K \cdot (c \wedge \bar{c}) \cdot \min\br{\sqrt{\frac{(c \wedge \bar{c}) V}{n}},(c \wedge \bar{c}) \cdot \frac{V}{nh}}
	\]
	where $K > 0$ is some absolute constant. 
\end{theorem}
When $h = 0$ (or being too small), we get a minimax lower bound of order $\br{c \wedge \bar{c}}^{\frac{3}{2}} \cdot \sqrt{\frac{V}{n}}$, and when $h = c \wedge \bar{c}$, we obtain a bound of the order $\br{c \wedge \bar{c}} \cdot \frac{V}{n}$. When $c=1/2$ in Theorem~\ref{main-theo-cost-classi}, we recover the standard minimax lower bounds for balanced binary classification (with zero-one loss function) presented in \cite[Theorem~4]{massart2006risk}:
\[
\underline{\mathcal{R}}_{\Delta \ell^{0-1}}^\star ~\geq~ K' \cdot \min\br{\sqrt{\frac{V}{n}},\frac{V}{nh}} ,
\]
for some constant $K' > 0$.

	\section{Conclusion}
\label{sec:conclusion}

We have investigated the influence of cost terms on the hardness of the cost-sensitive classification problem (Theorem~\ref{main-theo-cost-classi}) by extending the minimax lower bound analysis for balanced binary classification \cite[Theorem~4]{massart2006risk}. 

It would be interesting to study the hardness of the following classification problem settings which are closely related to the binary cost-sensitive classification problem that we considered in this paper:
\begin{enumerate}
	\item Cost-sensitive classification with example dependent costs (\cite{zadrozny2001learning,zadrozny2003cost}).
	\item Binary classification problem w.r.t.\ generalized performance measures (\cite{koyejo2014consistent}) such as arithmetic, geometric and harmonic means of the true positive and true negative rates. These measures are more appropriate for imbalanced classification problem (\cite{cardie1997improving,elkan2001foundations}) than the usual classification accuracy. 
\end{enumerate}

	\bibliographystyle{plain}
	\bibliography{example_paper}
	
	\appendix
	\section{Proofs}
\label{sec:proofs}

\begin{replemma}{cost-sensitive-lemma}
	\label{repthm:cost-sensitive-lemma}
	Consider the binary classification problem $\br{\Theta,\br{\mathcal{X} \times \bc{-1,1}}^n,\mathcal{F},\ell_{d_{c}}}$. For any $f \in \mathcal{F}$ and $c \in (0,1)$,
	\[
	\Delta \ell_{d_{c}} \br{\theta,f} ~=~ \frac{1}{2} \cdot \Ee{\textnormal{\textsf{X}}\sim M_\theta}{\abs{\eta_\theta \br{\textnormal{\textsf{X}}} - c} \cdot \abs{f\br{\textnormal{\textsf{X}}} - f_\theta^\star \br{\textnormal{\textsf{X}}}}} ,
	\]
	where $f_\theta^\star$ is given by \eqref{bayes-classifier-lc}.
\end{replemma}

\begin{proof}
	Consider a fixed $x \in \mathcal{X}$. Recall that 
	\[
	f_\theta^\star\br{x} = \argmin_{f \in \bc{-1,1}^\mathcal{X}} {\Ee{\textsf{Y} \sim \eta_\theta(x)}{d_{c}\br{\textsf{Y},f\br{x}}}} = \mathrm{sign}\br{\eta_\theta \br{x} - c} .
	\]
	Therefore $\inf_{f \in \bc{-1,1}^\mathcal{X}}{\Ee{\textsf{Y} \sim \eta_\theta(x)}{d_{c}\br{\textsf{Y},f\br{x}}}} = \Ee{\textsf{Y} \sim \eta_\theta(x)}{d_{c}\br{\textsf{Y},f_\theta^\star\br{x}}}$.
	This implies
	\begin{align*}
	& \Ee{\textsf{Y} \sim \eta_\theta(x)}{d_{c}\br{\textsf{Y},f\br{x}}} - \Ee{\textsf{Y} \sim \eta_\theta(x)}{d_{c}\br{\textsf{Y},f_\theta^\star\br{x}}} \\
	~=~& \bar{c} \, \eta_\theta \br{x} \ind{f(x) \neq 1} + c \, \bar{\eta_\theta \br{x}} \ind{f(x) \neq -1} \\
	& - \bc{\bar{c} \, \eta_\theta \br{x} \ind{f_\theta^\star(x) \neq 1} + c \, \bar{\eta_\theta \br{x}} \ind{f_\theta^\star(x) \neq -1}} \\
	~=~& \ind{f(x) \neq f_\theta^\star(x)} \abs{\eta_\theta \br{x} - c} \\
	~=~& \frac{1}{2} \cdot \abs{f(x) - f_\theta^\star(x)} \cdot \abs{\eta_\theta \br{x} - c} .
	\end{align*}
	Then the proof is completed by noting that
	\begin{align*}
	\ell_{d_{c}} \br{\theta,f} - \ell_{d_{c}} \br{\theta,f_\theta^\star} ~=~& \Ee{\textsf{X} \sim M_\theta}{\Ee{\textsf{Y} \sim \eta_\theta\br{\textsf{X}}}{d_{c}\br{\textsf{Y},f\br{\textsf{X}}}} - \Ee{\textsf{Y} \sim \eta_\theta\br{\textsf{X}}}{d_{c}\br{\textsf{Y},f_\theta^\star\br{\textsf{X}}}}} \\
	~=~& \frac{1}{2} \, \Ee{\textsf{X} \sim M_\theta}{\abs{f\br{\textsf{X}} - f_\theta^\star\br{\textsf{X}}} \cdot \abs{\eta_\theta \br{\textsf{X}} - c}} .
	\end{align*}
\end{proof}

\begin{replemma}{sub-add-f-div}
	\label{repthm:sub-add-f-div}
	For all collections of distributions $P_i , Q_i \in \mathcal{P} \br{\mathcal{O}_i}$, $i \in [k]$
	\[
	d_{\mathrm{TV}} \br{\bigotimes_{i=1}^k P_i , \bigotimes_{i=1}^k Q_i} ~\leq~ \sum_{i=1}^{k}{d_{\mathrm{TV}} \br{P_i ,Q_i}} ,
	\]	
	and 
	\[
	\mathrm{He}^2 \br{\bigotimes_{i=1}^k P_i , \bigotimes_{i=1}^k Q_i} ~\leq~ \sum_{i=1}^{k}{\mathrm{He}^2 \br{P_i ,Q_i}} .
	\]
\end{replemma}

\begin{proof}
	Firstly $\br{P,Q} \mapsto d_{\mathrm{TV}} \br{P,Q}$ is a metric. Thus 
	\begin{align*}
	& d_{\mathrm{TV}}\br{\bigotimes_{i=1}^k P_i , \bigotimes_{i=1}^k Q_i} \\
	~=~& d_{\mathrm{TV}} \br{P_1 \otimes \br{\bigotimes_{i=2}^k P_i} , Q_1 \otimes \br{\bigotimes_{i=2}^k Q_i}} \\ 
	~\leq~& d_{\mathrm{TV}} \br{P_1 \otimes \br{\bigotimes_{i=2}^k P_i} , Q_1 \otimes \br{\bigotimes_{i=2}^k P_i}} + d_{\mathrm{TV}} \br{Q_1 \otimes \br{\bigotimes_{i=2}^k P_i} , Q_1 \otimes \br{\bigotimes_{i=2}^k Q_i}} \\ 
	~=~& d_{\mathrm{TV}} \br{P_1 , Q_1} + d_{\mathrm{TV}} \br{\bigotimes_{i=2}^k P_i , \bigotimes_{i=2}^k Q_i} ,
	\end{align*}
	where the second line follows by definition, the third follows from the triangle inequality and the forth is easily verified from the definition of $d_\mathrm{TV}\br{\cdot,\cdot}$. To complete the proof proceed inductively. 
	
	Let $\mu$ be a product measure on $\mathcal{O}_1 \times \mathcal{O}_2$, written as $\mu = \mu_1 \otimes \mu_2$, where $\mu_i := \mu \circ \pi_i$ denotes the image measure of the projection $\pi_i:\mathbb{R}^2 \ni \br{x_1,x_2} \mapsto \pi_i\br{x_1,x_2} = x_i$ w.r.t.\ $\mu$. Also let $P = P_1 \otimes P_2$, and $Q = Q_1 \otimes Q_2$. Define $p := \frac{dP}{d\mu}$, $q := \frac{dQ}{d\mu}$, $p_1 := \frac{dP_1}{d\mu_1}$, $p_2 := \frac{dP_2}{d\mu_2}$, $q_1 := \frac{dQ_1}{d\mu_1}$, and $q_2 := \frac{dQ_2}{d\mu_2}$. Then, by Tonelli's theorem, 
	\begin{align*}
	1 - \frac{1}{2} \mathrm{He}^2 \br{P,Q} ~=~& \int{\sqrt{pq} d\mu} \\
	~=~& \int{\sqrt{p_1 q_1} d\mu_1} \cdot \int{\sqrt{p_2 q_2} d\mu_2} \\
	~=~& \br{1 - \frac{1}{2} \mathrm{He}^2 \br{P_1,Q_1}} \cdot \br{1 - \frac{1}{2} \mathrm{He}^2 \br{P_2,Q_2}} .
	\end{align*}
	Thus we have 
	\begin{align*}
	\mathrm{He}^2 \br{P,Q} ~=~& 2 - 2 \br{1 - \frac{1}{2} \mathrm{He}^2 \br{P_1,Q_1}} \cdot \br{1 - \frac{1}{2} \mathrm{He}^2 \br{P_2,Q_2}} \\
	~=~& \mathrm{He}^2 \br{P_1,Q_1} + \mathrm{He}^2 \br{P_2,Q_2} - \frac{1}{2} \mathrm{He}^2 \br{P_1,Q_1} \mathrm{He}^2 \br{P_2,Q_2} \\
	~\leq~& \mathrm{He}^2 \br{P_1,Q_1} + \mathrm{He}^2 \br{P_2,Q_2} .
	\end{align*}
	To complete the proof proceed the above process iteratively. 
\end{proof}

\begin{repproposition}{lecam-method}
	\label{repthm:lecam-method}
	For any $c \in \br{0,1}$, the minimax risk $\underline{\mathcal{R}}_\rho^\star$ (given by \eqref{minimax-risk-parameter}) of the parameter estimation problem with (pseudo metric) loss function $\rho:\Theta \times \Theta \rightarrow \mathbb{R}$ is bounded from below as follows:
	\[
	\underline{\mathcal{R}}_\rho^\star ~\geq~ \sup_{\theta \neq \theta'}{\bc{\rho\br{\theta,\theta'} \cdot \br{c \wedge \bar{c} - \mathbb{I}_{f_c}\br{P_{\theta}^n,P_{\theta'}^n}}}} ,
	\]
	where $f_c$ is given by \eqref{primitive-fdiv-tent-func}.
\end{repproposition}

\begin{proof}
	Let $c \in \br{0,1}$ be arbitrary but fixed. Consider any two fixed parameters $\theta,\theta' \in \Theta$ s.t. $\theta \neq \theta'$ and an arbitrary estimator $\hat{\theta}:\mathcal{O}^n \rightsquigarrow \Theta$. Let $P_\theta^n$, and $P_{\theta'}^n$ (associated probability densities can be written as $dP_\theta^n$ and $dP_{\theta'}^n$) be the probability measures induced by $\theta$ and $\theta'$ respectively. For an arbitrary (but fixed) set of observations $o_1^n \in \mathcal{O}^n$, when $\bar{c} \cdot dP_{\theta'}^n\br{o_1^n} \geq c \cdot dP_\theta^n\br{o_1^n}$, we have
	\begin{align}
	& c \cdot dP_\theta^n\br{o_1^n} \Ee{\tilde{\theta} \sim \hat{\theta}(o_1^n)}{\rho(\theta,\tilde{\theta})} + \bar{c} \cdot dP_{\theta'}^n\br{o_1^n} \Ee{\tilde{\theta} \sim \hat{\theta}(o_1^n)}{\rho(\theta',\tilde{\theta})} \nonumber \\
	~=~& c \cdot dP_\theta^n\br{o_1^n} \Ee{\tilde{\theta} \sim \hat{\theta}(o_1^n)}{\rho(\theta,\tilde{\theta}) + \rho(\theta',\tilde{\theta})} + \br{\bar{c} \cdot dP_{\theta'}^n\br{o_1^n} - c \cdot dP_\theta^n\br{o_1^n}} \Ee{\tilde{\theta} \sim \hat{\theta}(o_1^n)}{\rho(\theta',\tilde{\theta})} \nonumber \\ 
	\stackrel{(i)}{~\geq~}& c \cdot dP_\theta^n\br{o_1^n} \Ee{\tilde{\theta} \sim \hat{\theta}(o_1^n)}{\rho(\theta,\tilde{\theta}) + \rho(\theta',\tilde{\theta})} \nonumber \\
	\stackrel{(ii)}{~\geq~}& c \cdot dP_\theta^n\br{o_1^n} \rho\br{\theta,\theta'}, \label{int-case-01}
	\end{align}
	where $(i)$ is due to $\bar{c} \cdot dP_{\theta'}^n\br{o_1^n} \geq c \cdot dP_\theta^n\br{o_1^n}$, and $(ii)$ is due to the triangle inequality. Similarly, for the case where $\bar{c} \cdot dP_{\theta'}^n\br{o_1^n} \leq c \cdot dP_\theta^n\br{o_1^n}$, we get
	\begin{equation}
	\label{int-case-02}
	c \cdot dP_\theta^n\br{o_1^n} \Ee{\tilde{\theta} \sim \hat{\theta}(o_1^n)}{\rho(\theta,\tilde{\theta})} + \bar{c} \cdot dP_{\theta'}^n\br{o_1^n} \Ee{\tilde{\theta} \sim \hat{\theta}(o_1^n)}{\rho(\theta',\tilde{\theta})} ~\geq~ \bar{c} \cdot dP_{\theta'}^n\br{o_1^n} \rho\br{\theta,\theta'}.
	\end{equation}
	By combining \eqref{int-case-01} and \eqref{int-case-02}, and summing over all $o_1^n \in \mathcal{O}^n$, we get, for any two $\theta, \theta' \in \Theta$ and any estimator $\hat{\theta}$, 
	\begin{align}
	& c \cdot \Ee{\textsf{O}_1^n \sim P_\theta^n}{\Ee{\tilde{\theta} \sim \hat{\theta}\br{\textsf{O}_1^n}}{\rho(\theta,\tilde{\theta})}} + \bar{c} \cdot \Ee{\textsf{O}_1^n \sim P_{\theta'}^n}{\Ee{\tilde{\theta} \sim \hat{\theta}\br{\textsf{O}_1^n}}{\rho(\theta',\tilde{\theta})}} \label{tempo-112} \\
	~\geq~& \rho\br{\theta,\theta'} \cdot \int{c dP_\theta^n \wedge \bar{c} dP_{\theta'}^n} \nonumber \\
	~=~& \rho\br{\theta,\theta'} \cdot \br{c \wedge \bar{c} - \mathbb{I}_{f_c}\br{P_\theta^n,P_{\theta'}^n}}, \label{sum-exp-link}
	\end{align}
	where the last equality follows from the definition of $c$-primitive $f$-divergences \eqref{eq-primitive-f-div-2}. By taking the supremum of both sides over the choices of $\theta, \theta'$ (since then the two terms in \eqref{tempo-112} collapse to one), we have
	\[
	\sup_{\theta \in \Theta}{\Ee{\textsf{O}_1^n \sim P_\theta^n}{\Ee{\tilde{\theta} \sim \hat{\theta}\br{\textsf{O}_1^n}}{\rho(\theta,\tilde{\theta})}}} ~\geq~ \sup_{\theta \neq \theta'}{\bc{\rho\br{\theta,\theta'} \cdot \br{c \wedge \bar{c} - \mathbb{I}_{f_c}\br{P_{\theta}^n,P_{\theta'}^n}}}}.
	\]
	The proof is completed by taking the infimum of both sides over $\hat{\theta}$. 
\end{proof}

\begin{repcorollary}{auxi-minimax-coro}
	\label{repthm:auxi-minimax-coro}
	Let $\pi$ be any prior distribution on $\Theta$, and let $\mu$ be any joint probability distribution of a random pair $(\theta,\theta')\in\Theta \times \Theta$, such that the marginal distributions of both $\theta$ and $\theta'$ are equal to $\pi$. Then for any $c \in \br{0,1}$, the minimax risk $\underline{\mathcal{R}}_\rho^\star$ (given by \eqref{minimax-risk-parameter}) of the parameter estimation problem is bounded from below as follows:
	\[
	\underline{\mathcal{R}}_\rho^\star ~\geq~ \Ee{(\theta,\theta') \sim \mu}{\rho\br{\theta,\theta'} \cdot \br{c \wedge \bar{c} - \mathbb{I}_{f_c}\br{P_{\theta}^n,P_{\theta'}^n}}} .
	\]
\end{repcorollary}

\begin{proof}
	First observe that for any prior $\pi$ 
	\[
	\underline{\mathcal{R}}_\rho^\star ~\geq~ \inf_{\hat{\theta}} \Ee{\theta \sim \pi}{\Ee{\textsf{O}_1^n \sim P_{\theta}^n}{\Ee{\tilde{\theta} \sim \hat{\theta}\br{\textsf{O}_1^n}}{\rho(\theta,\tilde{\theta})}}} ,
	\]
	since the minimax risk can be lower bounded by the Bayesian risk. Then by taking expectation of both sides of \eqref{sum-exp-link} w.r.t $\mu$ and using the fact that, under $\mu$, both $\theta$ and $\theta'$ have the same distribution $\pi$, the proof is completed. 
\end{proof}

\begin{reptheorem}{assouad-method}
	\label{repthm:assouad-method}
	Let $d \in \mathbb{N}$, $\Theta=\bc{-1,1}^d$ and $\rho = \rho_{\mathrm{Ha}}$, where the Hamming distance $\rho_{\mathrm{Ha}}$ is given by \eqref{eq-hamming-distance}. Then for any $c \in \br{0,1}$, the minimax risk of the parameter estimation problem satisfies
	\[
	\underline{\mathcal{R}}_{\rho_{\mathrm{Ha}}}^\star ~\geq~ d \br{c \wedge \bar{c} - \underset{\theta,\theta': \rho_{\mathrm{Ha}}\br{\theta,\theta'}=1}{\max} \mathbb{I}_{f_c}\br{P_{\theta}^n,P_{\theta'}^n}} .
	\]
\end{reptheorem}

\begin{proof}
	Recall that $\rho_{\mathrm{Ha}} (\theta,\theta') = \sum_{i=1}^{d}{\rho_i (\theta,\theta')}$, where $\rho_i (\theta,\theta') := \ind{\theta_i \neq \theta'_i}$, and each $\rho_i$ is a pseudo metric. Let $\pi (\theta) = \frac{1}{2^d}, \forall{\theta \in \bc{-1,1}^d}$. Also for each $i \in \bs{d}$, let $\mu_i$ be the distribution in $\Theta \times \Theta$ such that any random pair $(\theta,\theta') \in \Theta \times \Theta$ drawn according to $\mu_i$ satisfies
	\begin{enumerate}
		\item $\theta \sim \pi$ 
		\item $\rho_i (\theta,\theta') = 1$, and $\rho_{\mathrm{Ha}} (\theta,\theta') = 1$ ($\theta$ and $\theta'$ differ only in the $i$-th coordinate). 
	\end{enumerate}
	Then the marginal distribution of $\theta'$ under $\mu_i$ is
	\[
	\sum_{\theta \in \bc{-1,1}^d}^{}{\mu_i (\theta,\theta')} = \frac{1}{2^d} \sum_{\theta \in \bc{-1,1}^d}^{}{\ind{\theta_i \neq \theta'_i \text{ and } \theta_j = \theta'_j , j \neq i}} = \frac{1}{2^d} = \pi (\theta') ,
	\]
	since by construction of $\mu$, $\rho_{\mathrm{Ha}} \br{\theta , \theta'} = 1$ and for each $\theta'$ there is only one $\theta$ that differs from it in a single coordinate. Now consider 
	\begin{align*}
	\underline{\mathcal{R}}_{\rho_{\mathrm{Ha}}}^\star \stackrel{(i)}{~\geq~}& \inf_{\hat{\theta}} \Ee{\theta \sim \pi}{\Ee{\textsf{O}_1^n \sim P_{\theta}^n}{\Ee{\tilde{\theta} \sim \hat{\theta}\br{\textsf{O}_1^n}}{\rho_{\mathrm{Ha}} \br{\theta,\tilde{\theta}}}}} \\
	\stackrel{(ii)}{~=~}& \inf_{\hat{\theta}} \sum_{i=1}^{d}{\Ee{\theta \sim \pi}{\Ee{\textsf{O}_1^n \sim P_{\theta}^n}{\Ee{\tilde{\theta} \sim \hat{\theta}\br{\textsf{O}_1^n}}{\rho_i \br{\theta,\tilde{\theta}}}}}} \\
	~\geq~& \sum_{i=1}^{d}{\inf_{\hat{\theta}} \Ee{\theta \sim \pi}{\Ee{\textsf{O}_1^n \sim P_{\theta}^n}{\Ee{\tilde{\theta} \sim \hat{\theta}\br{\textsf{O}_1^n}}{\rho_i \br{\theta,\tilde{\theta}}}}}} \\
	\stackrel{(iii)}{~\geq~}& \sum_{i=1}^{d}{\Ee{(\theta,\theta') \sim \mu_i}{\rho_i\br{\theta,\theta'} \cdot \br{c \wedge \bar{c} - \mathbb{I}_{f_c}\br{P_{\theta}^n,P_{\theta'}^n}}}} \\
	\stackrel{(iv)}{~=~}& \sum_{i=1}^{d}{\Ee{(\theta,\theta') \sim \mu_i}{\br{c \wedge \bar{c} - \mathbb{I}_{f_c}\br{P_{\theta}^n,P_{\theta'}^n}}}} \\
	~\geq~& \sum_{i=1}^{d}{\underset{\theta,\theta': \rho_{\mathrm{Ha}}\br{\theta,\theta'}=1}{\min}{\br{c \wedge \bar{c} - \mathbb{I}_{f_c}\br{P_{\theta}^n,P_{\theta'}^n}}}} \\
	~=~& d \br{c \wedge \bar{c} - \underset{\theta,\theta': \rho_{\mathrm{Ha}}\br{\theta,\theta'}=1}{\max} \mathbb{I}_{f_c}\br{P_{\theta}^n,P_{\theta'}^n}},
	\end{align*}
	where $(i)$ is due to the fact that the minimax risk is lower bounded by the Bayesian risk, $(ii)$ is due to $\rho_{\mathrm{Ha}} (\theta,\theta') = \sum_{i=1}^{d}{\rho_i (\theta,\theta')}$, $(iii)$ is by Corollary~\ref{auxi-minimax-coro}, and $(iv)$ is by the fact that $\rho_i (\theta,\theta')=1$ under $\mu_i$ for every $i$.
\end{proof}

\begin{repcorollary}{assouad-corro}
	\label{repthm:assouad-corro}
	Let $\mathcal{O}$ be some set and $c \in [0,1]$. Define 
	\[
	\mathcal{P} \br{\mathcal{O}} := \bc{P_{\theta} \in \mathcal{P} \br{\mathcal{O}} : \theta \in \bc{-1,1}^d}
	\]
	be a class of probability measures induced by the parameter space $\Theta = \bc{-1,1}^d$. Suppose that there exists some function $\alpha: \bs{0,1} \rightarrow \mathbb{R}_+$, such that 
	\[
	\mathrm{He}^2 \br{P_{\theta},P_{\theta'}} \leq \alpha\br{c} , \quad \text{ if } \rho_{\mathrm{Ha}}\br{\theta,\theta'} = 1 ,
	\] 
	i.e. the two probability distributions $P_{\theta}$ and $P_{\theta'}$ (associated with the two parameters $\theta$ and $\theta'$ which differ  only in one coordinate) are sufficiently close w.r.t.\ Hellinger distance. Then the minimax risk of the parameter estimation problem with parameter space $\Theta = \bc{-1,1}^d$ and the loss function $\rho = \rho_{\mathrm{Ha}}$ is bounded below by
	\begin{equation}
	\label{assouad-practical-eq-1}
	\underline{\mathcal{R}}_{\rho_{\mathrm{Ha}}}^\star ~\geq~ d \cdot (c \wedge \bar{c}) \cdot \br{1 - \sqrt{\alpha\br{c} n}} .
	\end{equation}
\end{repcorollary}

\begin{proof}
	For any two $\theta,\theta' \in \Theta$ with $\rho_{\mathrm{Ha}}\br{\theta,\theta'} = 1$, we have
	\begin{align*}
	\mathbb{I}_{f_c}\br{P_{\theta}^n,P_{\theta'}^n} &~\leq~ \br{c \wedge \bar{c}} \cdot \mathrm{He} \br{P_{\theta}^n,P_{\theta'}^n} \\
	&~\leq~ \br{c \wedge \bar{c}} \cdot \sqrt{\sum_{i=1}^{n}{\mathrm{He}^2 \br{P_{\theta},P_{\theta'}}}} \\
	&~\leq~ \br{c \wedge \bar{c}} \cdot \sqrt{\alpha \br{c} n} \\
	\end{align*}
	Substituting this bound into Theorem~\ref{assouad-method} completes the proof.
\end{proof}

\begin{reptheorem}{main-theo-cost-classi}
	\label{repthm:main-theo-cost-classi}
	Let $\mathcal{F}$ be a VC class of binary-valued functions on $\mathcal{X}$ with VC dimension (refer Appendix~\ref{sec:vc-dimension-sec}) $V \geq 2$. Then for any $n \geq V$ and any $h \in [0,c \wedge \bar{c}]$, the minimax risk \eqref{cost-minmax-risk-eq} of the cost-sensitive binary classification problem $\br{\Theta_{h,\mathcal{F}}, \br{\mathcal{X} \times \bc{-1,1}}^n, \mathcal{F}, \ell_{d_{c}}}$ is lower bounded as follows: 
	\[
	\underline{\mathcal{R}}_{\Delta \ell_{d_c}}^\star ~\geq~ K \cdot (c \wedge \bar{c}) \cdot \min\br{\sqrt{\frac{(c \wedge \bar{c}) V}{n}},(c \wedge \bar{c}) \cdot \frac{V}{nh}}
	\]
	where $K > 0$ is some absolute constant. 
\end{reptheorem}

\begin{proof}
	Instantiate $\Theta = \mathcal{A} = B := \bc{-1,1}^{V-1}$, $\mathcal{O} = \mathcal{X} \times \bc{-1,1}$, and $A=\hat{b}$ in the general learning task. Then the resulting parameter estimation problem can be represented by $\br{B,\mathcal{O}^n,B,\rho_{\mathrm{Ha}}}$. Let $\mathcal{P} \br{\mathcal{O}^n} := \bc{P_b^n \in \mathcal{P} \br{\mathcal{O}}^n:b\in B}$ be the class of probability measures induced by the parameter space $B$. Then the minimax risk of this problem w.r.t.\ Hamming distance $\rho_{\mathrm{Ha}}$ is given by 
	\[
	\underline{\mathcal{R}}_{\rho_{\mathrm{Ha}}}^\star ~=~ \inf_{\hat{b}} \max_{b \in B} \Ee{\textsf{O}_1^n \sim P_b^n}{\Ee{b' \sim \hat{b}\br{\textsf{O}_1^n}}{\rho_{\mathrm{Ha}}\br{b,b'}}} .
	\]
	
	Observe that $\mathbb{P}_{P_b}\bs{\textsf{X}=x,\textsf{Y}=y} = \mathbb{P}_{P_b}\bs{\textsf{X}=x} \cdot \mathbb{P}_{P_b}\bs{\textsf{Y}=y | \textsf{X}=x}$ for $b \in B$ (by Bayes rule). For simplicity, we will write $\mathbb{P}_{P_b} \bs{\cdot}$ as $\mathbb{P}_b \bs{\cdot}$. Now we will construct these distributions.
	
	\textbf{Construction of marginal distribution }$\mathbb{P}_b\bs{\textsf{X}=x}, x \in \mathcal{X}$: Since $\mathcal{F}$ is a \textit{VC class} with \textit{VC dimension} $V$, $\exists \bc{x_1,...,x_V} \subset \mathcal{X}$ that is shattered, i.e. $\text{for any } \beta \in \bc{-1,1}^V, \exists f \in \mathcal{F} \text{ s.t. } f(x_i)=\beta_i, \forall{i \in \bs{V}}$. Given $p \in \bs{0,1/(V-1)}$, for each $b \in B$, let
	\begin{equation}
	\label{pb-marginal-eq}
	\mathbb{P}_b \bs{\textsf{X}=x} = \begin{cases}
	p, & \text{if } x=x_i \text{ for some } i \in \bs{V-1} \\
	1-(V-1)p, & \text{if } x=x_V \\
	0, & \text{otherwise}
	\end{cases}
	\end{equation}
	A particular value for $p$ will be chosen later. 
	
	\textbf{Construction of conditional distribution }$\mathbb{P}_b\bs{\textsf{Y}=y|\textsf{X}=x}, y \in \bc{-1,1}, x \in \mathcal{X}$: For each $b \in B$, let
	\begin{equation}
	\label{cost-regression}
	\eta_b \br{x} := \mathbb{P}_b \bs{\textsf{Y}=1 | \textsf{X}=x} = \begin{cases}
	c-h, & \text{if } x=x_i \text{ for some } i \in \bs{V-1}, \text{ and } b_i = -1 \\
	c+h, & \text{if } x=x_i \text{ for some } i \in \bs{V-1}, \text{ and } b_i = 1 \\
	0, & \text{otherwise} .
	\end{cases}
	\end{equation}
	Then the corresponding Bayes classifier can be given as follows:
	\begin{equation}
	\label{bayes-opt-func-proof}
	f_b^\star(x) = \begin{cases}
	-1, & \text{if } x=x_i \text{ for some } i \in \bs{V-1}, \text{ and } b_i = -1 \\
	1, & \text{if } x=x_i \text{ for some } i \in \bs{V-1}, \text{ and } b_i = 1 \\
	-1, & \text{otherwise}
	\end{cases}
	\end{equation}
	Now we show that $\bc{P_b : b \in B} \subseteq \bc{P_{\theta} : \theta \in \Theta_{h,\mathcal{F}}}$. First of all, from \eqref{cost-regression} we see that $\abs{\eta_b \br{x} - c} \geq h$ for all $x$ (indeed, $\abs{\eta_b \br{x} - c} = h$ when $x \in \bc{x_1,...,x_{V-1}}$, and $\abs{\eta_b \br{x} - c} = c$ otherwise). Second, because $\bc{x_1,...,x_V}$ is shattered by $\mathcal{F}$, there exists at least one $f \in \mathcal{F}$, such that $f_b^\star (x) = f (x)$ for all $x \in \bc{x_1,...,x_V}$. Thus, we get $B \subset \Theta_{h,\mathcal{F}}$.
	
	\textbf{Reduction to Parameter Estimation Problem: }
	We start with the following observation 
	\[
	\underline{\mathcal{R}}_{\Delta \ell_{d_c}}^\star ~\geq~ \inf_{\hat{f}}{\max_{b \in B}{\Ee{\bc{\br{\textsf{X}_i,\textsf{Y}_i}}_{i=1}^n \sim P_b^n}{\Ee{f \sim \hat{f}\br{\bc{\br{\textsf{X}_i,\textsf{Y}_i}}_{i=1}^n}}{\ell_{d_c} \br{b,f} - \ell_{d_c} \br{b,f_b^\star}}}}} ,
	\]
	since $B \subset \Theta_{h,\mathcal{F}}$. Define $M_\theta \br{x} := \mathbb{P}_{\theta} \bs{\textsf{X} = x}$, and $\eta_\theta\br{x} := \mathbb{P}_\theta \bs{\textsf{Y}=1 \mid \textsf{X}=x}$, for $x \in \mathcal{X}$. By Lemma~\ref{cost-sensitive-lemma}, for any classifier $f:\mathcal{X} \rightarrow \bc{-1,1}$ and any $\theta \in \Theta_{h,\mathcal{F}}$, we have 
	\[
	\ell_{d_c} \br{\theta,f} - \ell_{d_c} \br{\theta,f_\theta^\star} ~=~ \frac{1}{2} \cdot \Ee{\textsf{X} \sim M_\theta}{\abs{\eta_\theta \br{\textsf{X}} - c} \cdot \abs{f\br{\textsf{X}} - f_\theta^\star \br{\textsf{X}}}} .
	\]
	If $\theta \in \Theta_{h,\mathcal{F}}$, then using the above equation and the margin condition \eqref{noise-condition} we get
	\[
	\ell_{d_c} \br{\theta,f} - \ell_{d_c} \br{\theta,f_\theta^\star} ~\geq~ \frac{h}{2} \cdot \Ee{\textsf{X} \sim M_\theta}{\abs{f\br{\textsf{X}} - f_\theta^\star \br{\textsf{X}}}} = \frac{h}{2} \cdot \norm{f-f_\theta^\star}_{L_1 \br{M_\theta}},
	\]
	where $\norm{f}_{L_1 \br{M_\theta}}$ is given by \eqref{eq-lp-norm-func} with $p=1$ and $\mu = M_\theta$. Since there is no confusion, we can simply drop $M_\theta$ and write the $L_1$ norm as $\norm{\cdot}_{L_1}$.  Hence we have
	\begin{align*}
	& \inf_{\hat{f}}{\max_{b \in B}{\Ee{\bc{\br{\textsf{X}_i,\textsf{Y}_i}}_{i=1}^n \sim P_b^n}{\Ee{f \sim \hat{f}\br{\bc{\br{\textsf{X}_i,\textsf{Y}_i}}_{i=1}^n}}{\ell_{d_c} \br{b,f} - \ell_{d_c} \br{b,f_b^\star}}}}} \\
	~\geq~& \frac{h}{2} \cdot \inf_{\hat{f}}{\max_{b \in B}{\Ee{\bc{\br{\textsf{X}_i,\textsf{Y}_i}}_{i=1}^n \sim P_b^n}{\Ee{f \sim \hat{f}\br{\bc{\br{\textsf{X}_i,\textsf{Y}_i}}_{i=1}^n}}{\norm{f-f_b^\star}_{L_1}}}}} .
	\end{align*}
	Define 
	\[
	b_f ~:=~ \argmin_{b \in B}{\norm{f-f_b^\star}_{L_1}}.
	\]
	Then for any $b \in B$,
	\[
	\norm{f_{b_f}^\star-f_b^\star}_{L_1} \leq \norm{f_{b_f}^\star-f}_{L_1} + \norm{f-f_b^\star}_{L_1} \leq 2 \norm{f-f_b^\star}_{L_1} ,
	\]
	where the first inequality is due to the triangle inequality and the second follows from the definitions of $b_f$ and $f_{\theta}^\star$. Thus we have 
	\begin{align*}
	& \inf_{\hat{f}}{\max_{b \in B}{\Ee{\bc{\br{\textsf{X}_i,\textsf{Y}_i}}_{i=1}^n \sim P_b^n}{\Ee{f \sim \hat{f}\br{\bc{\br{\textsf{X}_i,\textsf{Y}_i}}_{i=1}^n}}{\ell_{d_c} \br{b,f} - \ell_{d_c} \br{b,f_b^\star}}}}} \\
	~\geq~& \frac{h}{4} \cdot \inf_{\hat{f}}{\max_{b \in B}{\Ee{\bc{\br{\textsf{X}_i,\textsf{Y}_i}}_{i=1}^n \sim P_b^n}{\Ee{f \sim \hat{f}\br{\bc{\br{\textsf{X}_i,\textsf{Y}_i}}_{i=1}^n}}{\norm{f_{b_f}^\star-f_b^\star}_{L_1}}}}} \\
	~=~& \frac{h}{4} \cdot \inf_{\hat{b}}{\max_{b \in B}{\Ee{\bc{\br{\textsf{X}_i,\textsf{Y}_i}}_{i=1}^n \sim P_b^n}{\Ee{b' \sim \hat{b}\br{\bc{\br{\textsf{X}_i,\textsf{Y}_i}}_{i=1}^n}}{\norm{f_{b'}^\star-f_b^\star}_{L_1}}}}} .
	\end{align*}
	For any two $b,b' \in B$, we have 
	\begin{align*}
	\norm{f_{b'}^\star-f_{b}^\star}_{L_1} &= \int_{\mathcal{X}}^{ }{\abs{f_{b'}^\star(x)-f_{b}^\star(x)}\mathbb{P}_b\bs{\textsf{X}=x} dx} \\
	&= p \sum_{i=1}^{V-1}{\abs{f_{b'}^\star(x_i)-f_{b}^\star(x_i)}} \\
	&= p \sum_{i=1}^{V-1}{\abs{b'_i-b_i}} \\
	&= 2 p \cdot \rho_{\mathrm{Ha}} \br{b,b'} ,
	\end{align*}
	where the second and third equalities are from \eqref{pb-marginal-eq} and \eqref{bayes-opt-func-proof}. Finally we get 
	\begin{align}
	& \inf_{\hat{f}}{\max_{b \in B}{\Ee{\bc{\br{\textsf{X}_i,\textsf{Y}_i}}_{i=1}^n \sim P_b^n}{\Ee{f \sim \hat{f}\br{\bc{\br{\textsf{X}_i,\textsf{Y}_i}}_{i=1}^n}}{\ell_{d_c} \br{b,f} - \ell_{d_c} \br{b,f_b^\star}}}}} \nonumber \\
	~\geq~& \frac{ph}{2} \cdot \inf_{\hat{b}}{\max_{b \in B}{\Ee{\bc{\br{\textsf{X}_i,\textsf{Y}_i}}_{i=1}^n \sim P_b^n}{\Ee{b' \sim \hat{b}\br{\bc{\br{\textsf{X}_i,\textsf{Y}_i}}_{i=1}^n}}{\rho_{\mathrm{Ha}} \br{b,b'}}}}} \nonumber \\
	~=~& \frac{ph}{2} \cdot \underline{\mathcal{R}}_{\rho_{\mathrm{Ha}}}^\star . \label{lower-bound-temp-eq}
	\end{align}
	
	\textbf{Applying Assouad's Lemma: }
	For any two $b,b' \in B$ we have
	\begin{align*}
	\mathrm{He}^2 \br{P_b,P_{b'}} &~=~ \sum_{i=1}^{V} \sum_{y \in \bc{-1,1}}^{}{\br{\sqrt{\mathbb{P}_b\bs{\textsf{X}=x_i,\textsf{Y}=y}}-\sqrt{\mathbb{P}_{b'}\bs{\textsf{X}=x_i,\textsf{Y}=y}}}^2} \\
	&~=~ p \sum_{i=1}^{V-1} \sum_{y \in \bc{-1,1}}^{}{\br{\sqrt{\mathbb{P}_b \bs{\textsf{Y}=y|\textsf{X}=x_i}}-\sqrt{\mathbb{P}_{b'}\bs{\textsf{Y}=y|\textsf{X}=x_i}}}^2} \\
	&~=~ p \sum_{i=1}^{V-1} {\ind{b_i \neq b'_i} \bc{\br{\sqrt{c-h} - \sqrt{c+h}}^2 + \br{\sqrt{\bar{c-h}} - \sqrt{\bar{c+h}}}^2}} \\
	&~=~ 2 p \br{1-\sqrt{c^2-h^2}-\sqrt{\bar{c}^2-h^2}} \rho_{\mathrm{Ha}}\br{b,b'} ,
	\end{align*}
	where the second and third equalities are from \eqref{pb-marginal-eq} and \eqref{cost-regression}. Thus the condition of the Corollary~\ref{assouad-corro} is satisfied with 
	\begin{equation*}
	2 p \br{1-\sqrt{c^2-h^2}-\sqrt{\bar{c}^2-h^2}} ~\leq~ 4 p \frac{h^2}{c \wedge \bar{c}} ~=:~ \alpha\br{c} ,
	\end{equation*}
	where the inequality is from Lemma~\ref{aux-bound-lemma}. Therefore we get
	\begin{align*}
	& \inf_{\hat{f}}{\max_{b \in B}{\Ee{\bc{\br{\textsf{X}_i,\textsf{Y}_i}}_{i=1}^n \sim P_b^n}{\Ee{f \sim \hat{f}\br{\bc{\br{\textsf{X}_i,\textsf{Y}_i}}_{i=1}^n}}{\ell_{d_c} \br{b,f} - \ell_{d_c} \br{b,f_b^\star}}}}} \\
	~\geq~& \frac{p h (V-1)}{2} \br{c \wedge \bar{c} - c \wedge \bar{c} \cdot \sqrt{4 p \frac{h^2}{c \wedge \bar{c}} n}} \\
	~=~& \frac{p h (V-1)}{2} \br{c \wedge \bar{c} - 2 h \sqrt{c \wedge \bar{c} \cdot p n}} ,
	\end{align*}
	where the first inequality is due to \eqref{assouad-practical-eq} and \eqref{lower-bound-temp-eq}. If we let $p=\frac{c \wedge \bar{c}}{9nh^2}$, then the term in the parentheses will be equal to $\frac{c \wedge \bar{c}}{3}$, and
	\[
	\inf_{\hat{f}}{\max_{b \in B}{\Ee{\bc{\br{\textsf{X}_i,\textsf{Y}_i}}_{i=1}^n \sim P_b^n}{\Ee{f \sim \hat{f}\br{\bc{\br{\textsf{X}_i,\textsf{Y}_i}}_{i=1}^n}}{\ell_{d_c} \br{b,f} - \ell_{d_c} \br{b,f_b^\star}}}}} ~\geq~ \frac{(c \wedge \bar{c})^2 (V-1)}{54nh} ,
	\]
	assuming that the condition $p \leq 1/(V-1)$ holds. This will be the case if $h \geq \sqrt{\frac{c \wedge \bar{c} (V-1)}{9n}}$. Therefore 
	\begin{equation}
	\label{cost-case-1}
	\underline{\mathcal{R}}_{\Delta \ell_{d_c}}^\star ~\geq~ \frac{(c \wedge \bar{c})^2 (V-1)}{54nh} , \quad \text{ if } h \geq \sqrt{\frac{c \wedge \bar{c} (V-1)}{9n}}.
	\end{equation}
	If $h \leq \sqrt{\frac{c \wedge \bar{c} (V-1)}{9n}}$, we can use the above construction with $\tilde{h} = \sqrt{\frac{c \wedge \bar{c} (V-1)}{9n}}$. Then, because $\Theta_{\tilde{h},\mathcal{F}} \subseteq \Theta_{h,\mathcal{F}}$ whenever $\tilde{h} \geq h$, we see that 
	\begin{align}
	& \inf_{\hat{f}}{\sup_{\theta \in \Theta_{h,\mathcal{F}}}{\Ee{\bc{\br{\textsf{X}_i,\textsf{Y}_i}}_{i=1}^n \sim P_{\theta}^n}{\Ee{f \sim \hat{f}\br{\bc{\br{\textsf{X}_i,\textsf{Y}_i}}_{i=1}^n}}{\ell_{d_c} \br{\theta,f} - \ell_{d_c} \br{\theta,f_\theta^\star}}}}} \nonumber \\
	~\geq~& \inf_{\hat{f}}{\sup_{\theta \in \Theta_{\tilde{h},\mathcal{F}}}{\Ee{\bc{\br{\textsf{X}_i,\textsf{Y}_i}}_{i=1}^n \sim P_{\theta}^n}{\Ee{f \sim \hat{f}\br{\bc{\br{\textsf{X}_i,\textsf{Y}_i}}_{i=1}^n}}{\ell_{d_c} \br{\theta,f} - \ell_{d_c} \br{\theta,f_\theta^\star}}}}} \nonumber \\
	~\geq~& \frac{(c \wedge \bar{c})^2 (V-1)}{54n\tilde{h}} \nonumber \\
	~=~& \frac{(c \wedge \bar{c})^{\frac{3}{2}}}{18} \sqrt{\frac{V-1}{n}}, \quad \text{ if } h \leq \sqrt{\frac{c \wedge \bar{c} (V-1)}{9n}}. \label{cost-case-2}
	\end{align}
	Observe that $\frac{(c \wedge \bar{c})^2 (V-1)}{54nh} \leq \frac{(c \wedge \bar{c})^{\frac{3}{2}}}{18} \sqrt{\frac{V-1}{n}}$ if $h \geq \sqrt{\frac{c \wedge \bar{c} (V-1)}{9n}}$, and $\frac{(c \wedge \bar{c})^2 (V-1)}{54nh} > \frac{(c \wedge \bar{c})^{\frac{3}{2}}}{18} \sqrt{\frac{V-1}{n}}$ otherwise. Then combining \eqref{cost-case-1} and \eqref{cost-case-2} completes the proof.
\end{proof}

\begin{lemma}
	\label{aux-bound-lemma}
	For $h \in \bs{0,c \wedge \bar{c}}$, we have $1-\sqrt{c^2-h^2}-\sqrt{\bar{c}^2-h^2} ~\leq~ 2 \frac{h^2}{c \wedge \bar{c}}$.
\end{lemma}

\begin{proof}
	Let $A = 1-\sqrt{c^2-h^2}-\sqrt{\bar{c}^2-h^2}$. Take series expansion of $A$ w.r.t.\ $h$ to get 
	\begin{align*}
	A ~=~& \frac{1}{2} \br{\frac{1}{c} + \frac{1}{\bar{c}}} h^2 + \frac{1}{8} \br{\frac{1}{c^3} + \frac{1}{\bar{c}^3}} h^4 + \frac{1}{16} \br{\frac{1}{c^5} + \frac{1}{\bar{c}^5}} h^6 + \frac{5}{128} \br{\frac{1}{c^7} + \frac{1}{\bar{c}^7}} h^8 \\
	& + \frac{7}{256} \br{\frac{1}{c^9} + \frac{1}{\bar{c}^9}} h^{10} + \frac{21}{1024} \br{\frac{1}{c^{11}} + \frac{1}{\bar{c}^{11}}} h^{12} + \frac{33}{2048} \br{\frac{1}{c^{13}} + \frac{1}{\bar{c}^{13}}} h^{14} + \cdots .
	\end{align*}
	Now $\frac{1}{2} \br{\frac{1}{c} + \frac{1}{\bar{c}}} \leq \frac{1}{c} \vee \frac{1}{\bar{c}} = \frac{1}{c \wedge \bar{c}}$ (since average is less than maximum). Thus
	\[
	A \leq h \bs{\frac{h}{c \wedge \bar{c}} + \frac{1}{4} \frac{h^3}{\br{c \wedge \bar{c}}^3} + \frac{1}{8} \frac{h^5}{\br{c \wedge \bar{c}}^5} + \frac{5}{64} \frac{h^7}{\br{c \wedge \bar{c}}^7} + \frac{7}{128} \frac{h^9}{\br{c \wedge \bar{c}}^9} + \frac{21}{512} \frac{h^{11}}{\br{c \wedge \bar{c}}^{11}} + \cdots}.
	\]
	Now we have 
	\begin{align*}
	h \leq c \wedge \bar{c} \implies& \frac{h}{c \wedge \bar{c}} \leq 1 \\
	\implies& \br{\frac{h}{c \wedge \bar{c}}}^\alpha \leq \frac{h}{c \wedge \bar{c}} , \forall{\alpha \geq 1} .
	\end{align*}
	Thus 
	\begin{align*}
	A \leq& h \bs{\frac{h}{c \wedge \bar{c}} + \frac{h}{c \wedge \bar{c}}\bc{\frac{1}{4} + \frac{1}{8} + \frac{5}{64} + \frac{7}{128} + \frac{21}{512} + \cdots}} \\
	\leq& h \bs{\frac{2 h}{c \wedge \bar{c}}} = \frac{2 h^2}{c \wedge \bar{c}} , 
	\end{align*}
	where the second inequality follows from the fact that (can be shown with the aid of computer or using the properties of gamma function)
	\[
	\frac{1}{4} + \frac{1}{8} + \frac{5}{64} + \frac{7}{128} + \frac{21}{512} + \cdots \leq 1 .
	\]
\end{proof}

\section{VC Dimension}
\label{sec:vc-dimension-sec}
A measure of complexity in learning theory should reflect which learning problems are inherently easier than others. The standard approach in statistical theory is to define the complexity of the learning problem through some notion of ``richness'', ``size'', ``capacity'' of the hypothesis class. 

The complexity measure proposed in \cite{vapnik1971uniform}, the \textit{Vapnik-Chervonenkis (VC) dimension} is a \textit{combinatorial} measure of the richness of classes of binary-valued functions when evaluated on samples. VC-dimension is independent of the underlying probability measure and of the particular sample, and hence is worst-case estimate with regard to these quantities.

We use the notation $x_1^m$ for a sequence $\br{x_1,\dots ,x_m} \in \mathcal{X}^m$, and for a class of binary-valued functions $\mathcal{F} \subseteq \bc{-1,1}^{\mathcal{X}}$, we denote by $\mathcal{F}_{|x_1^m}$
the \textit{restriction} of $\mathcal{F}$ to $x_1^m$:
\[
\mathcal{F}_{|x_1^m} = \bc{\br{f\br{x_1},\dots ,f\br{x_m}} \mid f \in \mathcal{F}} .
\]
Define the \textit{$m$-th shatter coefficient} of $\mathcal{F}$ as follows:
\[
\mathbb{S}_m \br{\mathcal{F}} := \max_{x_1^m \in \mathcal{X}^m}{\abs{\mathcal{F}_{|x_1^m}}} .
\]

\begin{definition}
	Let $\mathcal{F} \subseteq \bc{-1,1}^{\mathcal{X}}$ and let $x_1^m = \br{x_1,\dots ,x_m} \in \mathcal{X}^m$. We say $x_1^m$ is shattered by $\mathcal{F}$ if $\abs{\mathcal{F}_{|x_1^m}} = 2^m$; i.e. if  $\forall{b \in \bc{-1,1}^m}, \exists{f_{b} \in \mathcal{F}} \text{ s.t. } \br{f_{b}\br{x_1},\dots ,f_{b}\br{x_m}} = b$. The \textit{Vapnik-Chervonenkis (VC) dimension} of $\mathcal{F}$, denoted by $\text{VCdim}\br{\mathcal{F}}$, is the cardinality of the largest set of points in $\mathcal{X}$ that can be shattered by $\mathcal{F}$:
	\[
	\text{VCdim}\br{\mathcal{F}} = \max{\bc{m \in \mathbb{N} \mid \mathbb{S}_m \br{\mathcal{F}} = 2^m}} .
	\]
	If $\mathcal{F}$ shatters arbitrarily large sets of points in $\mathcal{X}$, then $\text{VCdim}\br{\mathcal{F}} = \infty$. If $\text{VCdim}\br{\mathcal{F}} < \infty$, we say that $\mathcal{F}$ is a VC class.
\end{definition}

\end{document}